\documentclass[11pt,a4paper]{article}
\usepackage[margin=25mm]{geometry}
\usepackage[T1]{fontenc}
\usepackage[utf8]{inputenc}
\usepackage{lmodern,microtype}
\usepackage{amsmath,amssymb,mathtools,amsthm}
\usepackage{booktabs,tabularx,array,longtable}
\usepackage{graphicx,caption,float}
\usepackage[section]{placeins}
\usepackage[numbers,sort&compress]{natbib}
\usepackage[hidelinks]{hyperref}
\usepackage[nameinlink,noabbrev]{cleveref}
\usepackage{fancyhdr}
\newcolumntype{Y}{>{\raggedright\arraybackslash}X}
\newcolumntype{L}[1]{>{\raggedright\arraybackslash}p{#1}}

\hypersetup{
 pdftitle={Many Perturbations, Shared Failure Dynamics: Tracing Direction-Sensitive Perturbation Influence in Language Models},
 pdfauthor={Zhibo Hu, Chen Wang, Yanfeng Shu, Hye-young Paik, Liming Zhu},
 pdfsubject={Revised research draft; 23 October 2026}
}
\newtheorem{proposition}{Proposition}
\newcommand{\Din}{D_{\mathrm{in}}}
\newcommand{\Dout}{D_{\mathrm{out}}}
\newcommand{\Gw}{G_{W,\epsilon}}
\newcommand{\Gz}{G_{W,0}}
\newcommand{\norm}[1]{\left\lVert #1\right\rVert_2}
\newcommand{\ip}[2]{\left\langle #1,#2\right\rangle}

\newcommand{\E}{\mathbb{E}}
\newcommand{\trans}{\mathrm{trans}}
\newcommand{\longi}{\mathrm{long}}
\newcommand{\sema}{\mathrm{semantic}}
\newcommand{\panel}[1]{\paragraph{#1}}
\graphicspath{{figures/}}
\title{\textbf{Many Perturbations, Shared Failure Dynamics:}\\Tracing Direction-Sensitive Perturbation Influence\\in Language Models}
\author{%
\begin{minipage}{\dimexpr\textwidth-2\tabcolsep\relax}
\centering
{\normalsize\bfseries
Zhibo Hu\textsuperscript{1,2}\quad
Chen Wang\textsuperscript{1,2}\quad
Yanfeng Shu\textsuperscript{2}\quad
Liming Dong\textsuperscript{2}\quad
Hye-young Paik\textsuperscript{1}\quad
Liming Zhu\textsuperscript{1,2}}\\[6pt]
{\normalsize\textsuperscript{1}The University of New South Wales\qquad
\textsuperscript{2}CSIRO technology}\\[2pt]
{\normalsize Australia}\\[6pt]
{\footnotesize\ttfamily
\{zhibo.hu, h.paik\}@unsw.edu.au\quad
\{Chen.Wang, Yanfeng.Shu, Liming.Dong, Liming.Zhu\}@csiro.au}
\end{minipage}%
}
\date{}
\begin{document}
\maketitle
\begin{abstract}
Why can a task-preserving rewrite turn a correct answer into a failure, while a larger internal change leaves the answer intact? We develop a direction-sensitive, readout-relative account of semantic fragility. Our central proposal is that heterogeneous language changes can share geometric and readout relationships even when their behaviorally consequential directions depend on the input. Clean decision reserve, the transformation of an incoming representation difference, and its alignment with a specified response provide complementary descriptions of this process.

An exact longitudinal--transverse decomposition exposes the directional information discarded by scalar transmission. A compact reserve-and-transmission model predicts failures across held-out perturbation families and items in Qwen3-32B and Gemma 4 31B, and transfers to ReClor. Matched interventions then show that orientation carries causal information beyond those scalar summaries. Attention-component experiments on disjoint original items, followed by answer-order controls, further support an advantage for the natural paired direction over matched random and same-answer-letter donor directions. Crucially, the same rewrite rarely fails consistently across option orders, and effective current-order corrections use a newly computed clean reference. Shared predictive regularity therefore coexists with context-dependent directional effects.

This account also distinguishes a localized intervention from the natural input change it probes. Path integration improves the description of a finite patch within a fixed context, while a substantial discrepancy from the full rewrite effect remains. The complementary fixed-CoT cue study extends the readout-relative perspective to influence not acknowledged in recorded reasoning: its readout association transfers across domains, although the positive-mean causal criterion fails and signed-effect prediction remains post hoc.

Together, these results locate the commonality across linguistic perturbations in relationships among reserve, transformation, and response. They support a framework for studying shared fragility without assuming an input-independent failure direction or equating local causal sensitivity with a complete explanation of natural failure.
\end{abstract}

\section{Introduction}
\label{sec:intro}
External changes to a prompt can alter a language model's internal computation without making their consequences immediately visible. A task-preserving rewrite may leave an answer unchanged or turn a correct answer into an incorrect one. An external cue may likewise affect a readout even when the recorded chain of thought does not acknowledge it. Linguistic robustness and chain-of-thought (CoT) monitorability usually study these settings separately \citep{wang2024rupbench,fu2025same,wang2026monitorbench}. We ask a common question: \emph{what determines whether an input-induced representation difference becomes consequential for a specified behavioral readout?}

Representation change alone does not answer this question. Two rewrites can produce large hidden-state displacements while only one changes the answer. Likewise, the absence of cue acknowledgement in written reasoning does not establish the absence of an internal effect. We therefore distinguish the size of an input-induced difference, its transformation through an internal window, and its relation to the response being measured. This distinction permits shared measurements without treating an answer error and an undisclosed cue influence as the same outcome.

\begin{figure}[htbp]
\centering
\includegraphics[width=0.96\linewidth]{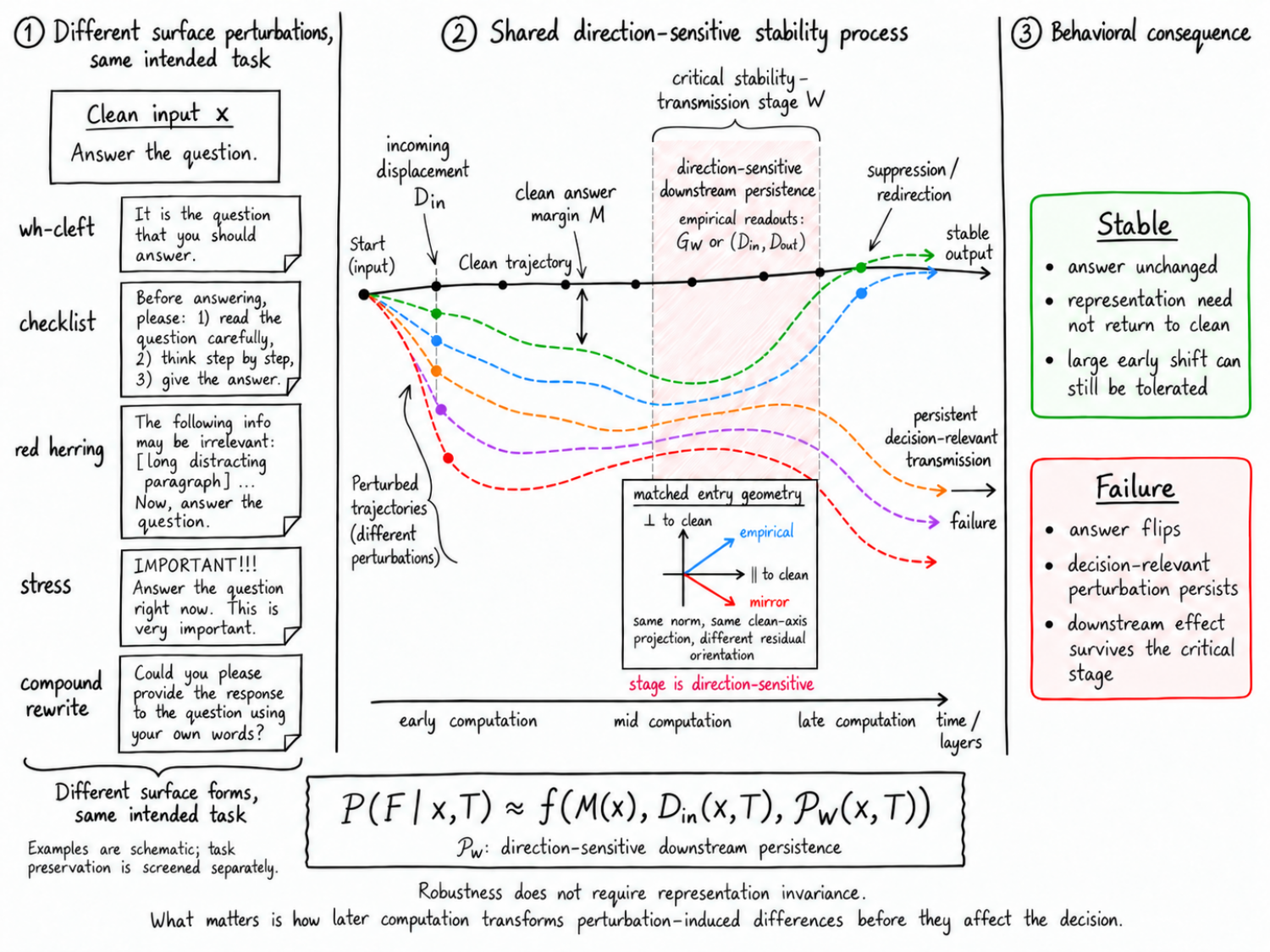}
\caption{Conceptual starting point: heterogeneous rewrites induce different internal trajectories through a model-specific window. The experiments refine the initial transmission account by testing orientation, readout alignment, and dependence on the current answer layout. Contraction is descriptive rather than a sufficient condition for stability; neither a universal direction nor a deterministic robustness certificate is assumed.}
\label{fig:conceptual}
\end{figure}

We begin with semantic fragility. On six RUPBench families paired across the same MMLU questions \citep{hendrycks2021mmlu,wang2024rupbench}, a compact description based on clean answer margin, incoming representation displacement, and local displacement gain predicts correct-to-incorrect changes across held-out families and items. It reaches AUROC 0.953 for Qwen3-32B and 0.910 for Gemma 4 31B \citep{yang2025qwen3,gemmateam2026gemma4}; frozen scorers retain useful ranking performance on independently constructed ReClor rewrites \citep{yu2020reclor}. The shared predictive relationship concerns decision reserve and transmission, not a shared surface trigger or hidden-state vector. Historical window selection also remains distinct from held-out score fitting.

Prediction does not establish scalar transmission as a sufficient mechanism. For paired differences $\delta_a$ and $\delta_b$ at the endpoints of a window, write
\begin{equation}
 \delta_b=(1+q_W)\delta_a+\rho_W\Din\hat r_W,
 \qquad \Din=\norm{\delta_a},\quad \hat r_W\perp\delta_a.
 \label{eq:intro-decomp}
\end{equation}
Here $q_W$ describes longitudinal change, $\rho_W$ transverse magnitude, and $\hat r_W$ transverse orientation. This exact identity exposes information discarded by scalar gain. In Qwen, failure-associated excess squared gain is largely transverse, and exchanging transverse orientation can alter behavior while preserving the receiver's scalar decomposition. A small fixed-rule ReClor replication retains that effect. In Gemma, entry interventions alter the naturally computed downstream state, while mirror controls with equal update norm, clean-axis projection, and resulting entry distance evolve differently. These experiments establish sensitivity to direction under their respective controls; they do not identify a universal restoring direction.

The new experiments pursue two gaps left by these results. First, a finite patch in one input's context is not the same operation as changing the input. On the original ReClor cohort, numerical path integration better describes the measured finite patch than a single-point gradient, but does not eliminate a larger discrepancy associated with the rest of the input-dependent computation. A quantitative local effect therefore cannot by itself explain the complete natural rewrite effect. Second, a promising local component must be evaluated with its selection history intact. An initial held-family component validation gives no clear all-cohort benefit. Exploratory failure-conditioned diagnostics then motivate a revised rule, which is frozen before inference on 250 disjoint original ReClor items. The rule targets the final-token output of Qwen's L54 attention component and strictly repairs 13/27 failures, compared with 0/108 matched random-direction draws. This new-item result supports a localized directional effect after the refinement; it does not replace the preceding negative all-cohort result.

A further experiment asks whether the paired direction's advantage can be reduced to answer-letter information or intervention size. We reuse the new cohort under four cyclic option orders and compare each current-order clean-minus-rewrite direction with random directions and cross-family discovery donors having the same gold answer letter. Controls match the update norm and resulting attention-component norm, rather than the full window's scalar decomposition. In the common-clean cohort, 37 failure conditions arise from 32 rewrites and 23 original items. The paired direction exceeds both control means in every failure condition, and strictly repairs 16/37, versus 10/148 same-letter donor and 0/148 random draws. However, 28/32 affected rewrites fail in only one order and none in all four. Every order also uses its own clean reference. The warranted conclusion is thus \emph{input- and readout-context-specific directional information beyond the tested local geometry and donor controls}; unchanged-vector transfer across answer layouts remains untested.

Direction must also be interpreted relative to a response. If $r$ is a transverse difference and $g$ the gradient of a chosen readout, $J_{\trans}=g^\top r$ supplies a signed local description. The complementary MonitorBench experiment \citep{wang2026monitorbench} holds the recorded CoT fixed and compares a standard cue with an exact-target lexical control. At positions predicting the target sequence, readout-aligned transverse attribution transfers observationally to held-out math and health domains. The frozen positive-mean causal criterion fails; a subsequent signed patch-effect analysis is correct on 73/82 bases (89.02\%) and remains explicitly post hoc. These existing results concern influence during fixed-CoT replay, rather than recovery of hidden reasoning or a validated detector.

Our contributions are fourfold:
\begin{itemize}
\item A compact cross-family predictive description of semantic fragility, followed by matched interventions demonstrating information omitted by scalar gain.
\item A finite-intervention analysis that separates improved path-based description of a local patch from explanation of the full natural input effect.
\item A component-level validation sequence that retains the initial negative result, tests a revised frozen rule on disjoint original items, and uses expanded answer-order and same-letter donor controls to establish the local scope of its directional advantage.
\item A readout-relative connection to the original fixed-CoT cue study, retaining its transferred association, failed mean-effect criterion, and post hoc signed result as distinct evidence.
\end{itemize}

\Cref{sec:theory,sec:setup} define the measurements and chronology. \Cref{sec:prediction,sec:semantic-causal} present predictive and causal semantic-fragility results, including the new local-window and answer-order studies. \Cref{sec:monitor,sec:monitor-results} retain the fixed-CoT methods and results. The discussion distinguishes shared descriptive relationships from context-independent directions, complete explanations of natural failures, and deployable monitors.

\section{Shared Geometry, Distinct Behavioral Questions}
\label{sec:theory}

\subsection{Paired computations and response-specific outcomes}
Let $P_0$ be a reference input and $P_1$ a contrasted input, with aligned residual states $h^0_{\ell,t}$ and $h^1_{\ell,t}$ at layer $\ell$ and measured position $t$. Set $\delta_{\ell,t}=h^1_{\ell,t}-h^0_{\ell,t}$. Axis I uses the clean question as $P_0$ and its rewrite as $P_1$, at the final prompt token. Axis II uses a lexical-control prompt as $P_0$ and a standard cue as $P_1$, at positions predicting a fixed target sequence after fixed generated CoT. ``Aligned'' here refers to those response-predicting positions; it does not assume equal prompt lengths or tokenwise equivalence throughout the inputs.

The decomposition below is shared, but the outcome is not. Axis I relates geometry to a gold margin and a correct-to-incorrect answer change. Axis II relates it to a target-score contrast and an intervention effect on that score. The target answer need not be correct, so a positive cue-score effect is not intrinsically an error or a harmful change. We do not apply the MMLU failure classifier as a hidden-CoT detector, and we do not invent a clean-margin analogue where the second protocol did not measure one. The two axes share measurements only where their definitions actually coincide.

\subsection{Clean reserve and the unsuccessful first-order boundary account}
For a multiple-choice item with correct option $y$, let $z_j(x)$ be the next-token logit of option $j$. Define
\begin{equation}
 m_j(x)=z_y(x)-z_j(x),\qquad M(x)=\min_{j\ne y}m_j(x).
 \label{eq:margin}
\end{equation}
The main failure analysis conditions on a correct clean answer. For a rewrite $T(x)$, failure means that the selected option becomes incorrect. Clean margin is an empirical measure of decision reserve, not an independent measure of semantic competence.

An initial approximation linearized each margin at an intermediate clean state, with $g_{j,\ell}=\partial m_j/\partial h_\ell$:
\begin{equation}
 \widehat m_j(T(x))=m_j(x)+g_{j,\ell}^{\top}\delta_\ell,
 \qquad R_\ell=\max_{j\ne y}\frac{-g_{j,\ell}^{\top}\delta_\ell}{m_j(x)}.
 \label{eq:R}
\end{equation}
At 50\% depth in Qwen, $R_\ell$ obtained strict held-family/held-item AUROC 0.683, below clean margin alone (0.869), and literal $R_\ell>1$ crossings were essentially absent. Adding observed mid-to-final displacement gain to clean margin and mid-layer displacement improved AUROC to 0.907. These diagnostics reject the adequacy of that particular single-position, mid-layer approximation; they do not show that local derivatives are universally uninformative. They motivated measuring the intervening trajectory rather than treating the initial displacement as already decision-equivalent.

\subsection{Historical scalar coordinates and the fitted risk score}
For the final prompt-token residual state at window endpoints $a<b$, set
\begin{equation}
 \Din=\norm{\delta_a},\qquad \Dout=\norm{\delta_b},\qquad
 \Gw=\frac{\Dout}{\Din+\epsilon}.
 \label{eq:scalar}
\end{equation}
The historical compact classifier uses the \emph{raw} features $(M,\Din,\Gw)$, training-set median imputation, and training-set standardization. With $z_k=(v_k-\mu_k)/\sigma_k$, its log-odds score is
\begin{equation}
 \Lambda=\beta_0+\beta_M z_M+\beta_D z_D+\beta_G z_G,
 \qquad p(F)=\operatorname{sigmoid}(\Lambda).
 \label{eq:score}
\end{equation}
There is no logarithmic transform in this compact score. Log-transformed features appear only in separately identified diagnostics, including the conditional within-item analysis. The fitted coefficient signs associate greater clean reserve with lower risk and larger displacement and transmission with higher risk. These associations are not monotonicity guarantees for arbitrary interventions.

For nonzero $\Din$, distinguish the unregularized geometric ratio $\Gz=\Dout/\Din$ from the stabilized ratio used by the scorer. $\Gz<1$ denotes net contraction; $\Gz>1$ denotes net expansion. The ratio is an observed property of a pair of trajectories, not an operator norm, Lipschitz constant, or robustness certificate \citep{tsuzuku2018lipschitz,muthukumar2023local}. In particular, a safe rewrite need not have $\Gz<1$.

\subsection{Exact decomposition of the window update}
Define the difference in accumulated window updates by
\begin{equation}
 u_W=\delta_b-\delta_a
     =(h_b^T-h_a^T)-(h_b^C-h_a^C).
 \label{eq:windowupdate}
\end{equation}
For $\Din>0$, its longitudinal coefficient and transverse residual are
\begin{equation}
 q_W=\frac{\ip{\delta_a}{u_W}}{\Din^2},\qquad
 r_W=u_W-q_W\delta_a,\qquad
 \rho_W=\frac{\norm{r_W}}{\Din}.
 \label{eq:decompose}
\end{equation}
When $r_W\ne0$, write $\hat r_W=r_W/\norm{r_W}$. Orthogonality gives $r_W\perp\delta_a$ and
\begin{align}
 \delta_b&=(1+q_W)\delta_a+r_W
          =(1+q_W)\delta_a+\rho_W\Din\hat r_W,\label{eq:decomp}\\
 \Gz^2&=(1+q_W)^2+\rho_W^2,\label{eq:gainidentity}\\
 \Gw^2&=\left(\frac{\Din}{\Din+\epsilon}\right)^2
           \big[(1+q_W)^2+\rho_W^2\big].\label{eq:epsilonidentity}
\end{align}
The factor in \cref{eq:epsilonidentity} matters for notation: an identity stated for $\Gz$ must not silently be presented as exact for the regularized scorer feature. The experiments numerically check their implemented decomposition; the algebraic identity itself does not depend on a trained model.

The earlier cancellation/innovation statistics are the same geometry in another form:
\begin{equation}
 C_W=-2q_W,\qquad I_W=\frac{\norm{u_W}^2}{\Din^2}=q_W^2+\rho_W^2,
 \qquad \Gz^2=1-C_W+I_W.
 \label{eq:cancellation}
\end{equation}
Positive $C_W$ denotes an update component opposing the incoming displacement, not successful behavioral stabilization. Transverse energy can outweigh that component. Likewise, a negative $q_W$ does not by itself imply contraction, and an increase in $\rho_W$ is not inherently harmful. When $\Din=0$, these normalized coordinates are undefined; when $r_W=0$, there is no defined transverse orientation to swap. Such cases require explicit numerical handling rather than an arbitrary direction.

\subsection{Why equal gain need not imply equal behavioral effect}
\begin{proposition}[Scalar geometry does not determine a readout]
Fix $\delta_a\ne0$, $q_W$, and $\rho_W>0$. All endpoints of the form
$(1+q_W)\delta_a+\rho_W\Din\hat r$, with unit $\hat r\perp\delta_a$, have the same displacement norm and gain. A linear readout with a nonzero component orthogonal to $\delta_a$ need not assign them the same value.
\end{proposition}
\begin{proof}
Orthogonality proves the common norm in \cref{eq:gainidentity}. For readout vector $g$, its endpoint change is
$(1+q_W)g^\top\delta_a+\rho_W\Din g^\top\hat r$. Choosing opposite transverse orientations aligned and anti-aligned with the nonzero transverse component of $g$ changes the sign of the second term without changing any of the scalar coordinates.
\end{proof}
This proposition is a mathematical possibility result, not evidence that a particular model uses a specific readout direction. The orientation-swap experiments test whether this omitted degree of freedom is consequential in the models and cohorts studied here.

\subsection{From orientation to signed readout alignment}
\label{sec:signed-theory}
A scalar response $\psi$ is locally sensitive to a direction $v$ through $\nabla\psi^\top v$. At one endpoint position, with gradient $g$ evaluated at a specified reference computation,
\begin{equation}
 J_{\mathrm{total}}=g^\top\delta_b
  =\underbrace{g^\top[(1+q_W)\delta_a]}_{J_{\longi}}
   +\underbrace{g^\top r_W}_{J_{\trans}}.
 \label{eq:readoutdecomp}
\end{equation}
The quantity is signed and response-dependent. For a gold margin, a positive contribution locally favors the correct answer relative to the chosen competitor; for target log probability it favors that target, whether or not the target is correct. Neither $\rho_W$ nor $\norm{r_W}$ supplies that sign. This is the conceptual link between the two research axes: magnitude describes how much the representations differ, whereas readout alignment describes the local consequence of a particular direction.

For a sequence-level response, let $H_b^0=(h^0_{b,t})_{t\in\mathcal I}$ collect the states to be patched and let $\phi$ map this tuple to the response, with the reference input, all unpatched layer-$b$ states, and any teacher-forced continuation fixed. For a transverse tuple $R=(r_t)_{t\in\mathcal I}$, define
\begin{align}
 C_R(\alpha)&=\phi(H_b^0+\alpha R)-\phi(H_b^0),\label{eq:general-patch}\\
 J_R&=\sum_{t\in\mathcal I}\left(\nabla_{h^0_{b,t}}\phi\right)^\top r_t,
 \qquad C_R'(0)=J_R.\label{eq:derivative}
\end{align}
The derivative identity holds where the reference computation is differentiable. Under suitable local smoothness, $C_R(\alpha)=\alpha J_R+O(\alpha^2)$. This elementary statement does not guarantee the sign or rank of $C_R(1)$: finite-step agreement is an empirical test. Nor does it imply $\E[C_R(1)]>0$. An accurate signed predictor can coexist with an average effect near zero when different examples have effects of different signs.

\panel{Scope of the local model.}
Axis II instantiates $\phi$ as mean target-sequence log probability at the lexical reference; its gradients and patches use the same target-prediction positions. This does not retroactively validate the failed mid-layer boundary approximation in \cref{eq:R}, which used a different reference, position, response, and perturbation scale. It also does not identify a special semantic circuit merely because a matched derivative predicts a small patch effect. The held-domain relation to the full prompt contrast is a separate empirical question.

A last-token residual state is not a closed dynamical state: other positions and subsequent attention can affect the result. The same limitation applies to a tuple restricted to target-prediction positions. Its patch effect need not equal the effect of changing the entire prompt. \Cref{app:context} formalizes the missing-context term. Residual downstream displacement after a clean last-token replacement identifies insufficiency of the partial state, not a unique contextual reinjection pathway.

\subsection{What is shared, and what is not}
The original hypothesis proposed approximate cross-family regularity,
\begin{equation}
 P(F\mid M,\Din,\Gw,T=t_1)\approx P(F\mid M,\Din,\Gw,T=t_2).
 \label{eq:sharedhyp}
\end{equation}
Held-family prediction and small incremental value of family indicators support this as a useful empirical approximation on the historical distribution. They do not prove conditional independence, sufficiency of the three scalars, or equality of underlying circuits. The refined hypothesis retains shared \emph{measurement coordinates} while allowing family-specific trajectories, model-specific window locations, and example-specific directions. The experiments distinguish scalar prediction, directional causal sensitivity, and readout-specific signed influence rather than treating them as interchangeable evidence.

\section{Experimental Design and Evidence Chronology}
\label{sec:setup}
\subsection{Historical paired MMLU cohort}
We pair 500 MMLU base questions with six published RUPBench transformations: \texttt{checklist}, \texttt{compound}, \texttt{it\_cleft}, \texttt{red\_herring}, \texttt{stress}, and \texttt{wh\_cleft}. All six variants are present before evaluation. Qwen answers 399 clean items correctly, yielding 2,394 variant rows and 193 failures; Gemma retains 428 items, 2,568 rows, and 185 failures. A family name is not evidence that every instance preserves the task; \cref{sec:audit} reports the separate validity sensitivity analysis.

Multiple-choice scoring uses the next-token logits of single-token labels A, B, C, and D. States are measured at the final prompt position. Qwen runs use non-thinking inference and bfloat16; Gemma uses its fixed replication configuration. Layer numbers denote the block-output probes in each runner, not universal normalized coordinates. Qwen has 64 blocks and probes L32, 40, 48, 50, 52, 54, and 56; Gemma has 60 blocks and probes L30, 38, 45, 47, 49, 51, and 52. The historical prospective predictor excludes the final LM-head input.

\subsection{Predictive evaluation and window selection}
The strict protocol holds out a family and base items simultaneously. For each held-out family, deterministic five-fold item splits exclude every variant of the test items from training, as well as the entire held-out family. Median imputation, standardization, and unweighted logistic fitting use only the training partition. We report the mean of six held-family AUROCs. Grouped out-of-fold and cross-model results are separate estimands.

The window scan is historical discovery: Qwen L52--54 and Gemma L38--45 were selected from the respective MMLU depth analyses. Selection was not independently nested inside every reported prediction fold. They were then fixed for the subsequent experiments that explicitly inherit them. Thus held-family/item prediction controls score fitting, but is not by itself a fully external test of window selection. Base-item bootstrap resampling preserves dependence among variants; neither variant rows nor random donor draws are independent sample units.

\subsection{Original ReClor construction and transfer studies}
The ReClor stimulus set contains lexical paraphrases, grammatical restructurings, and irrelevant-context additions. Of 750 requested generation slots, 746 candidates were imported and AI-screened; 707 were approved, 10 rejected for grammar/rules, 18 rejected for semantic changes, and 11 marked uncertain. The 707 retained variants cover 250 base questions. The text and screening decisions were frozen before the initial target evaluation, without selecting by candidate answer or mechanism score. These are AI-screen-passed rewrites, not certified semantic equivalents.

Both model evaluations reuse this same text set and then condition on their own clean-correct items. Frozen MMLU-to-ReClor \emph{prediction} is distinct from the later fixed-rule \emph{intervention} replication. The latter reuses the ReClor cohort and inherits windows, numerical rules, and interventions from MMLU; it is not relabeled as a fresh, untouched-cohort confirmation. The Qwen endpoint parameterization was developed after the initial ReClor ratio-score result but fitted only on historical MMLU. Both Gemma parameterizations were frozen before Gemma ReClor inference; the stimuli had already been evaluated on Qwen.

\subsection{Intervention eligibility, controls, and analysis units}
Every intervention contrasts an edited computation with its replayed baseline and checks that the hook is removed afterward. Numerical feasibility is determined from natural states and requested edits, before new edited logits. Direction-swap groups require the swap and all four matched random controls to satisfy the frozen tolerances; incomplete groups are not analyzed selectively. Bidirectional analyses use only items eligible in both receiver directions. Magnitude and orientation feasibility are screened separately.

For the historical ReClor transverse-swap replication, the two prespecified primary contrasts were the main-window signed bidirectional orientation effect and the $\alpha=0.10$ equal-dose transverse-minus-longitudinal response. Its protocol specifies 10,000 item-bootstrap resamples and nominal 97.5\% intervals for each contrast, corresponding to a Bonferroni nominal 95\% familywise allocation. Small sample size and feasibility selection still limit precision. No dose increase, tolerance relaxation, or outcome-dependent seed replacement is used to rescue an infeasible experiment.

\subsection{New finite-effect and component experiments}
The later Qwen experiments retain the original 707 approved rewrites for finite-effect replay, natural window tracing, and component discovery/validation. They use 601 clean-correct pairs from 213 original items. These reused-input analyses are distinct from the earlier seven-item transverse-swap replication: the readout replay patches a full L54 residual state, while the component study patches an attention or MLP output and allows the suffix to recompute. Their scalar summaries fix the worst competitor selected from the clean logits. The initial component validation targets the average effect over all eligible rewrites. A subsequent failure-conditioned rule is an exploratory refinement informed by that result, frozen before the new-item experiment.

The new acquisition excludes all 250 old original items and their clean-token hashes, leaving 250 other labelled ReClor items. Fixed generation and AI review produce 702 approved rewrites. The target remains Qwen3-32B; generation and review use the same Qwen3.8-27B checkpoint. The later option-order experiment reuses these rewrites under four cyclic layouts, with clean references recomputed per layout. It switches to minimum gold margin for its effect score, making its numerical effects distinct from the earlier fixed-competitor margins. All new component intervals use 2,000 original-item bootstrap resamples. Primary-family allocations are nominal 98.33\% for three contrasts, or 99.17\% for six order-control contrasts. Twenty failure items per family is a frozen interpretive threshold, not a power guarantee; none of the new failure cohorts meets it. \Cref{sec:finite_reclor,sec:window_components,sec:fresh-validation,sec:order-controls} provide the complete denominators and chronology.

\subsection{CoT-monitorability cohorts and frozen chronology}
The second axis uses Qwen3-32B cue-answer tasks from MonitorBench. Logic is the development domain (50 raw bases); its reported causal analysis contains 93 continuations from 25 eligible bases. The later cross-domain protocol starts from 28 math and 65 health bases. It freezes the L52--54 window, standard-minus-lexical contrast, target-position attribution, patch doses, controls, and success rule before the new-domain mechanism results are inspected. The analyzed nondisclosed cohort has 348 continuations from 82 bases: 86/20 for math and 262/62 for health. These are continuation/base counts, not model-accuracy numerators and denominators.

The rule is frozen for the math/health \emph{mean-effect} test, which fails. The later use of signed patch-effect accuracy as the refined headline outcome is post hoc. A zero sign threshold and a previously computed gradient do not make that retrospective choice confirmatory. A prepared Gemma replication is not counted as completed evidence. \Cref{sec:monitor,sec:monitor-results} give the measurement and outcomes separately.

\begin{table}[htbp]
\centering\small
\caption{Evidence chronology. ``Frozen'' refers to the specified stage, not to a single preregistration covering the full research program.}
\label{tab:chronology}
\begin{tabularx}{\linewidth}{@{}L{0.25\linewidth}L{0.25\linewidth}Y@{}}
\toprule
Study & What was fixed & Status and remaining exposure\\
\midrule
MMLU geometry & Per-fold fitting and held families/items & Window selection is historical discovery.\\
ReClor prediction & MMLU score and inherited window & Qwen endpoint comparison is post hoc; Gemma scorers frozen before its inference.\\
ReClor transverse swaps & Inherited windows, edits, numerical rules & Historical fixed-rule replication on reused stimuli; seven primary direction items.\\
Finite replay / window trace & L52--54, response and replay protocol & Reused old items; partial-state and finite-precision limits remain.\\
Initial component validation & Source-family/base selection, matched controls & All three all-pair intervals include zero; revised estimand is separate.\\
New-item component test & Failure-conditioned selector and L54.attn & Frozen before 250 disjoint items; low failure counts and correlated AI review.\\
Option-order controls & Component, donors, matching and six contrasts & Reuses new-item cohort; current-order references, not vector transport.\\
CoT: logic & Development analysis & Hypothesis-development evidence, not external confirmation.\\
CoT: math/health & Contrast, window, attribution, patch rule, mean-effect gates & Observational gate passes; original causal gates fail.\\
CoT: signed prediction & Zero-threshold rule on completed outputs & Post hoc refined estimand; independent confirmation pending.\\
\bottomrule
\end{tabularx}
\end{table}

\panel{Scope of the integrated evidence.}
The studies were completed in successive stages, not one jointly preregistered experiment. Historical discovery, frozen prediction, fixed-rule replication on previously used stimuli, and post hoc analyses are labeled separately. The CoT-monitorability axis has its own contrast, eligibility rule, and outcome; it is not pooled with semantic-rewrite failure labels. The results retained in this working revision and their source-recovery status are documented in \cref{app:provenance} and the accompanying evidence ledger.

\section{Axis I: Predicting Semantic Fragility}
\label{sec:prediction}
This axis begins with discrete correctness and then asks which internal differences distinguish failed from successful rewrites. The scalar predictor is a useful baseline description; subsequent direction-sensitive experiments test what it leaves out.
\subsection{A late predictive transition in Qwen}
\Cref{tab:historical,fig:qwen-depth} summarize the historical results. Clean margin is already informative (AUROC 0.869), and displacement adds little in the first half of Qwen. The compact margin--incoming-displacement--gain model improves from 0.877 at L48 to 0.906 at L50, 0.923 at L52, and 0.953 at L54. Its largest adjacent increase is approximately 0.0295 at the L52--54 transition; performance falls to 0.916 at L56. The region is a late band with a predictive peak, not a single universally decisive block.

\begin{table}[htbp]
\centering\small
\caption{Historical MMLU cohorts and strict mean held-family/held-item AUROC. Windows were selected in the historical depth analyses. The AUROC columns and grouped out-of-fold results in the text are distinct estimands.}
\label{tab:historical}
\begin{tabular}{lrrrrcc}
\toprule
Model & Bases & Variants & Failures & Failure rate & Margin & Compact \\
\midrule
Qwen3-32B &399&2,394&193&8.06\%&0.869&0.953\\
Gemma 4 31B &428&2,568&185&7.20\%&0.787&0.910\\
\bottomrule
\end{tabular}
\end{table}

\begin{figure}[htbp]
\centering\includegraphics[width=0.82\linewidth]{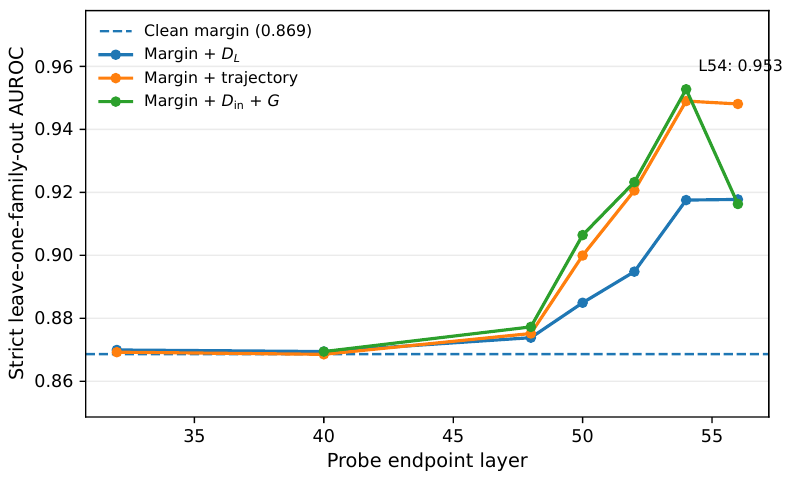}
\caption{Historical Qwen depth profile, retained from the earlier manuscript. The compact score peaks at L54, using $M$, $D_{52}$, and $G_{52\to54,\epsilon}$. This scan selects the later intervention window; it is not a separately nested window-selection test.}
\label{fig:qwen-depth}
\end{figure}

Held-family AUROCs at that window are 0.959 for checklist, 0.959 for compound, 0.950 for it-cleft, 0.952 for red-herring, 0.945 for stress, and 0.951 for wh-cleft. The grouped item-disjoint out-of-fold AUROC is 0.952. Adding family intercepts and interactions changes it to 0.951 and slightly worsens log loss, from 0.1493 to 0.1499. These results support a useful shared predictor on the historical distribution, rather than a formal equality of conditional failure laws.

\subsection{Within-question discrimination beyond clean difficulty}
Among the 399 clean-correct Qwen items, 93 fail under at least one rewrite and 92 have both failed and successful variants. Their 625 failed--successful pairs hold the clean question fixed. Incoming displacement has pairwise concordance 0.680, local gain 0.778, and their standardized sum 0.814; a family-only out-of-fold model gives 0.621.

Exact conditional logistic regression removes all item-constant susceptibility, including clean margin. A one-standard-deviation increase in log incoming displacement has odds ratio 2.54 ($p=1.37\times10^{-9}$); the corresponding log-gain odds ratio is 5.46 ($p=1.05\times10^{-16}$). Both associations remain large after adding family nuisance indicators. Within-item label permutations give $p<0.001$ for each geometric quantity and their sum. This log-feature conditional analysis is distinct from the raw-feature compact predictor in \cref{eq:score}.

\subsection{Retrospective semantic-validity sensitivity analysis}
\label{sec:audit}
A separate AI-assisted audit considered 100 stratified base questions and all six variants, or 600 pairs. Screening information included the original and transformed question, options, and gold answer, but not target predictions, failure labels, geometry, or frozen risk. Criteria included negation, quantification, predicate--argument relations, conditional scope, and task-relevant information changes.

The available screening labels retain all 100 checklist and 100 stress examples and 61 red-herring examples, but reject the compound, it-cleft, and wh-cleft examples under the conservative task-preservation criterion. Of the audited bases, 84 are in Qwen's historical clean-correct cohort, giving 504 rows. Applying the labels without refitting leaves 221 retained rows and 19 failures. \Cref{tab:audit} shows that discrimination changes little.

\begin{table}[htbp]
\centering\small
\caption{Frozen Qwen score after retrospective semantic filtering. These audited rows overlap the historical fitting data. The table is a sensitivity analysis under preliminary AI-screening labels, not held-out generalization or independent human verification. Within-item concordance uses only failed--successful pairs.}
\label{tab:audit}
\begin{tabular}{lrrrrrr}
\toprule
Subset & Rows & Failures & Failure rate & AUROC & AP & Concordance\\
\midrule
All audited historical &504&40&7.94\%&0.957&0.593&0.866\\
Screen-passed &221&19&8.60\%&0.953&0.589&0.875\\
Screen-rejected &283&21&7.42\%&0.961&0.599&0.842\\
Checklist + stress &168&16&9.52\%&0.950&0.606&0.917\\
\bottomrule
\end{tabular}
\end{table}

For checklist and stress, the failure-rate Wilson 95\% interval is [5.95\%, 14.91\%]. Mean predicted risk is 7.96\%, compared with 9.52\% observed. Within-item concordances on the screen-passed and checklist--stress subsets are based on only 24 and 12 informative pairs, respectively.

The available audit history also records default decisions with specified exception overrides in an earlier annotation-generation workflow. Populated annotation fields therefore do not independently establish exhaustive row-by-row review. We treat the screening labels as preliminary and the analysis as conditional on them. Neither the family-level acceptance counts nor the stable AUROC prove semantic equivalence of every retained rewrite. A fresh human audit would strengthen the semantic interpretation.

\subsection{Replication at a different depth in Gemma}
\label{sec:gemma-predict}
Gemma's margin baseline is 0.787. Both its best trajectory model and the compact L38--45 score reach approximately 0.910. Adding family identity to the strongest shared geometry reduces the reported item-disjoint AUROC from 0.910 to 0.905. The transition occurs earlier than Qwen's, supporting portability of the measurement structure rather than a universal layer fraction (\cref{fig:gemma-depth}).

\begin{figure}[htbp]
\centering\includegraphics[width=0.82\linewidth]{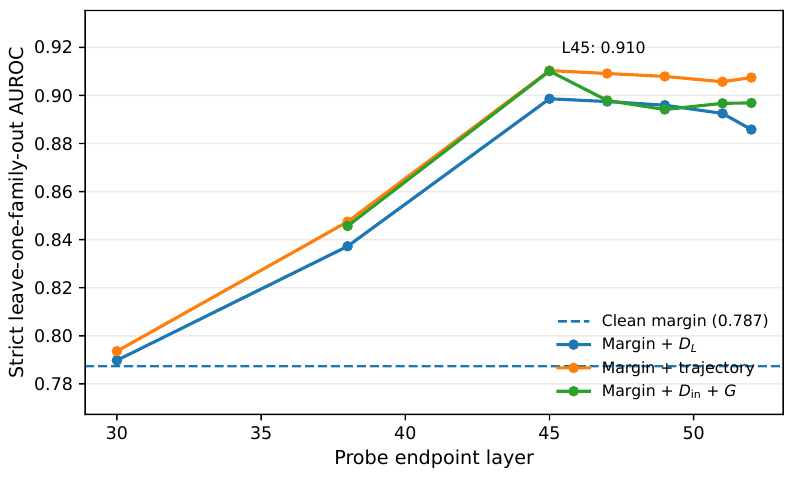}
\caption{Historical Gemma depth profile, retained from the earlier manuscript. The useful transition occurs around L38--45 of 60 blocks. It does not coincide with Qwen's normalized window.}
\label{fig:gemma-depth}
\end{figure}

On the 378 bases that both models answer correctly in clean form, item-level fragility correlates with Spearman $\rho=0.440$ ($p=2.71\times10^{-19}$). With Qwen's normalized L52--54 window imposed on both models, the two transfer directions yield AUROC 0.811 (Qwen to Gemma) and 0.891 (Gemma to Qwen). Using each model's historically selected window yields 0.900 and 0.943. Target failure labels are not used to refit the transferred score, but target normalization uses unlabeled target features, and the Gemma window was selected using Gemma's historical outcomes. This is transductive zero-target-label score transfer, not source-only external validation.

\subsection{Frozen MMLU-to-ReClor prediction}
\label{sec:reclor-pred}
\Cref{tab:reclor} makes explicit results previously reported only in the Introduction. Qwen retains 214 clean-correct bases and 604 variants, with 15 failures (2.48\%). Gemma retains 240 bases and 680 variants, with 14 failures (2.06\%). The same 707 approved variants are the starting stimulus set for both models; evaluated rows differ because their clean-correct sets differ.

\begin{table}[htbp]
\centering\small
\caption{Frozen MMLU-to-ReClor transfer. ``Ratio'' uses $(M,\Din,\Gw)$; ``endpoint'' uses $(M,\Din,\Dout)$. The Qwen endpoint comparison was developed post hoc after the initial ratio result, but uses only historical MMLU fitting labels. Both Gemma forms were frozen before its ReClor evaluation. A dash denotes a value not recovered in the available summary, not zero.}
\label{tab:reclor}
\begin{tabular}{lrrrccc}
\toprule
Model & Bases & Variants & Failures & $-M$ AUROC & Ratio AUROC & Endpoint AUROC\\
\midrule
Qwen3-32B &214&604&15&0.9561&0.9261&0.9782\\
Gemma 4 31B &240&680&14&--&0.9340&0.9689\\
\bottomrule
\end{tabular}
\end{table}

The original Qwen ratio score therefore does \emph{not} outperform clean margin on this transfer set: its AUROC is 0.926089 versus 0.956140 for $-M$. It also overestimates average risk, 0.099894 predicted versus 0.024834 observed. The endpoint form reaches AUROC 0.978155 and Brier score 0.018875. Gemma's ratio and endpoint AUROCs are 0.934041 and 0.968897; the archived endpoint-minus-ratio interval includes zero. Both ratio scorers overestimate mean risk, so ranking transfer should not be equated with calibrated failure probabilities.

The comparison has a simple statistical explanation that must not be mistaken for new information. With finite valid features and fixed $\epsilon$,
\begin{equation}
 \Dout=(\Din+\epsilon)\Gw.
 \label{eq:invertible}
\end{equation}
Thus $(\Din,\Gw)$ and $(\Din,\Dout)$ are invertibly related. Linear logistic models in the two coordinate systems nevertheless impose different functional restrictions and regularization geometry: one uses a ratio and the other an implicit displacement--gain interaction. Consistently, Gemma's historical MMLU comparison shows little separation: mean LOFO AUROC 0.909775 versus 0.911182 and grouped out-of-fold AUROC 0.906117 versus 0.906597. This is not a formal equivalence test. The appropriate conclusion is useful transfer with parameterization and calibration limits, not universal superiority of one three-variable score.

\subsection{Natural trajectories: what the gain association contains}
\label{sec:natural}
The scalar score identifies where transmission becomes predictive but does not explain the observed increase in gain. We apply \cref{eq:decompose,eq:cancellation} to natural Qwen trajectories in four fixed windows: L48--50, L50--52, L52--54, and L54--56. The archived analysis divides the 92 mixed-outcome items into 38 development and 54 evaluation items. This is a within-historical-cohort analysis, not a new dataset evaluation.

For each evaluation item, failed and successful rewrites are compared within the same base question. \Cref{tab:natural} reports failure-minus-safe differences. Transverse energy is higher for failures through L48--54, with the largest reported difference at L52--54. At that window, $\Delta\rho_W^2=0.3843$, compared with $\Delta\Gz^2=0.4408$. Their ratio is approximately 87.2\%. This is the ratio of two mean descriptive contrasts, \emph{not} a causal mediation fraction or explained variance.

\begin{table}[htbp]
\centering\small
\caption{Natural Qwen geometry, evaluation split of 54 mixed-outcome base items. Entries are within-item failure-minus-safe means with archived item-bootstrap 95\% intervals. The final window shows that the pattern is not a monotone increase toward the output.}
\label{tab:natural}
\begin{tabular}{lccc}
\toprule
Window & $\Delta q_W$ & $\Delta\rho_W^2$ & $\Delta\Gz^2$\\
\midrule
L48--50 &0.0396 [0.0247, 0.0548]&0.2516 [0.1629, 0.3455]&0.3250 [0.2145, 0.4388]\\
L50--52 &0.0221 [0.0148, 0.0295]&0.1743 [0.1218, 0.2316]&0.2172 [0.1559, 0.2836]\\
L52--54 &0.0271 [0.0179, 0.0360]&0.3843 [0.2665, 0.5093]&0.4408 [0.3198, 0.5715]\\
L54--56 &0.0083 [$-$0.0008, 0.0171]&$-$0.0446 [$-$0.0727, $-$0.0172]&$-$0.0267 [$-$0.0631, 0.0048]\\
\bottomrule
\end{tabular}
\end{table}

The corresponding L52--54 cancellation difference is $-0.0541$ [${-}0.0720$, ${-}0.0358$], while innovation differs by $+0.3867$ [0.2689, 0.5115]. This is consistent with both less opposing longitudinal response and greater transverse difference, rather than a purely longitudinal failure process. Since the longitudinal contribution to squared gain is $(1+q_W)^2$, its group difference must not be replaced by $(\Delta q_W)^2$.

A separately reported incoming-displacement/length-matched comparison retains 32 items and 38 matched pairs at L52--54. It gives $\Delta\rho_W^2=0.3098$ [0.1550, 0.4744] and $\Delta\Gz^2=0.3519$ [0.1794, 0.5310], a descriptive ratio of approximately 88.0\%. Matching does not remove all directional or linguistic confounds, but it reduces the explanation that the entire difference reflects incoming magnitude or prompt length.

At L54--56, the transverse-energy difference reverses and the squared-gain interval includes zero. The proposed critical region should therefore be understood as a model- and contrast-dependent band of informative directional processing, not an inevitable late-layer amplification cascade. Natural geometry motivates the direction and magnitude experiments below; it cannot substitute for them.

\subsection{Prospective use without candidate-answer feedback}
\label{sec:search-bridge}
The same frozen scalar score guides a behavior-blind search on 241 previously unused clean-correct MMLU items. The final beam fails on 20.02\% of candidates versus 2.49\% in frozen low-risk controls, an 8.04-fold enrichment. Its AUROC remains 0.886 within the final beam. The full search, reveal, and selection accounting is retained in \cref{app:search}. No reusable new linguistic family was frozen or discovered. These results show that an incomplete scalar summary can prioritize fragile natural rewrites; they do not make it sufficient under interventions that deliberately change an omitted direction. The exploratory MASC procedure-steering result is retained separately in \cref{app:masc}, rather than counted as a third main axis.

\section{Axis I: Directional Causal Tests}
\label{sec:semantic-causal}
Prediction does not identify what a window naturally computes. We therefore distinguish Qwen endpoint correction, fixed-scalar transverse orientation exchange, and Gemma entry-geometry controls followed by a naturally computed endpoint. The new studies then test finite readout replay, interventions within the Qwen window, and matched component directions across new items and answer layouts. Their patch locations and held-fixed quantities differ. Together they test whether direction matters; they do not provide interchangeable estimates of one universal causal effect.

\subsection{Historical Qwen Causal Tests: Dose, Direction, and Layer}
\label{sec:causal-qwen}
\subsubsection{Endpoint contraction and reverse amplification}
The historical intervention replaces the perturbed final prompt-token state at L54 with
\begin{equation}
 h_{54}^{(\alpha)}=h_{54}^{C}+\alpha(h_{54}^{T}-h_{54}^{C}).
 \label{eq:alpha}
\end{equation}
Here $\alpha=1$ is a no-op, $\alpha<1$ contracts the endpoint difference, and $\alpha>1$ amplifies it. All computation through L53 is unchanged, including the incoming L52 measurement. The primary cohort comprises 187 failed variants from mixed-outcome items and 187 safe variants matched by family.

\begin{table}[htbp]
\centering\small
\caption{Historical Qwen endpoint contraction on 187 failed variants. Gold-margin changes are relative to the replayed unedited variant, unlike the swap-minus-random contrasts reported later.}
\label{tab:dose}
\begin{tabular}{rrr}
\toprule
Endpoint scale $\alpha$ & Mean gold-margin change & Failure rescue\\
\midrule
1.00 &0.000&0.0\%\\
0.75 &1.114&17.6\%\\
0.50 &2.270&51.3\%\\
0.25 &3.48&69.0\%\\
0.00 &4.793&80.7\%\\
\bottomrule
\end{tabular}
\end{table}

Recovery is graded (\cref{tab:dose}). All six family-specific contraction slopes are positive, with their reported bootstrap intervals above zero; a family-label permutation test does not detect slope heterogeneity ($p=0.186$). This absence of detected heterogeneity is not proof of equal effects. On safe variants, amplification to $\alpha=1.25$ reduces the mean margin by 0.281 and induces failure in 4.8\%; $\alpha=1.5$ reduces the margin by 0.529 and induces failure in 9.63\%. Contraction is the stronger causal observation, with amplification providing a reverse-direction check.

Full last-token replacement does not rescue every failure. Other positions and the rest of the perturbed computation remain unchanged. Moreover, this patch directly changes the endpoint: it does not show that the natural window implements the same correction or that scalar gain mediates all behavioral change.

\subsubsection{Equal-norm direction controls}
For 72 failed variants, 12 per family, define $c=h_{54}^{C}-h_{54}^{T}$. At strength $s=0.5$, compare $sc$ with an equal-norm isotropic random update and an equal-norm random update orthogonal to $c$. Each case has eight draws of each control. The reported maximum norm mismatch is below 0.1\%, and orthogonal cosine is numerically zero.

Clean-directed correction improves the gold margin by 2.153 and rescues 51.4\% of failures. Random updates improve the margin by 0.013 and rescue 2.78\%; orthogonal updates change it by $-0.005$ and rescue 1.56\%. The item-cluster bootstrap contrasts are 2.140 [1.729, 2.594] for clean minus random and 2.158 [1.747, 2.589] for clean minus orthogonal. Own-clean correction beats the mean random update on 70 of 72 cases and the mean orthogonal update on all 72.

The result establishes direction specificity for this intervention, not one shared vector across examples. Each clean correction is computed from that example's own counterpart. It also leaves a geometric confound for comparing own-clean against empirical donors: at a given step norm, own-clean most efficiently reduces distance to the reference. Gemma's mirror experiment addresses that issue separately.

\subsubsection{Absolute-norm-matched layer localization}
For the same 72-case design and $\ell\in\{40,48,50,52,54,56\}$, let $c_\ell=h_\ell^C-h_\ell^T$, $D_\ell=\norm{c_\ell}$, and
\begin{equation}
 D_{\min}=\min_\ell D_\ell,\qquad
 v_\ell=0.5D_{\min}\frac{c_\ell}{D_\ell}.
 \label{eq:normlayer}
\end{equation}
Every layer receives the same absolute norm within an example, in its own clean-correction direction. The minimum is at L40 in all 72 cases, and the maximum realized within-case norm spread is approximately 0.15\%.

Mean margin improvement is 0.023 at L40, 0.038 at L48, 0.155 at L50, 0.269 at L52, 0.465 at L54, and 0.417 at L56 (\cref{fig:norm}). The largest adjacent increase is L54 minus L52: 0.196, with item-cluster 95\% interval [0.143, 0.258]. L54 also exceeds L56 by 0.049 [0.007, 0.096]. All six families have a positive L54-minus-L52 contrast. Binary rescue contrasts are not significant under this deliberately small update budget; continuous margin change is the primary outcome.

\begin{figure}[htbp]
\centering\includegraphics[width=0.70\linewidth]{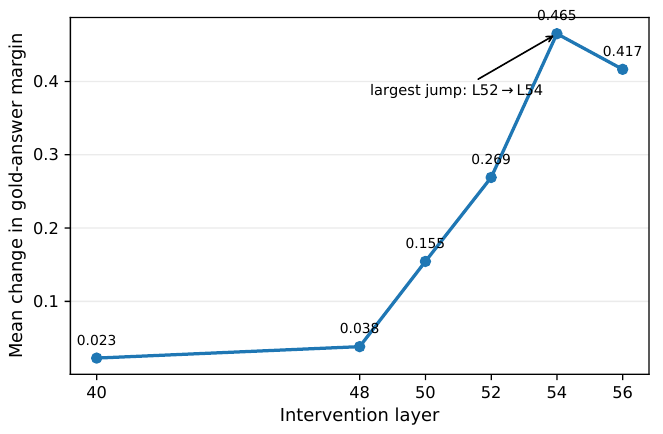}
\caption{Absolute-norm-matched Qwen correction, retained from the earlier manuscript. The natural clean--variant distance is not the update budget: the applied norm is held fixed within each case across layers. Corrective leverage peaks at L54 rather than increasing monotonically to L56.}
\label{fig:norm}
\end{figure}

This result limits a generic ``later layers are easier to patch'' explanation within the tested set. It localizes corrective sensitivity for the specified intervention, not a uniquely identified endogenous repair module. The predictive window was selected before these interventions, but the studies share the historical benchmark and therefore are not independent dataset replications.

\subsection{Separating Transverse Orientation from Magnitude}
\label{sec:orientation}
\subsubsection{Endpoint orientation swaps with fixed scalar geometry}
For a receiver rewrite $i$ and an opposite-outcome donor rewrite $j$ of the same base question, derive a donor transverse direction in the receiver's transverse plane. With $P_i=I-\delta_{a,i}\delta_{a,i}^{\top}/D_i^2$, the nondegenerate direction is the normalized projection of the donor transverse vector. The intended endpoint is
\begin{equation}
 \delta_{b,i}^{\mathrm{swap}}
  =(1+q_i)\delta_{a,i}+\rho_iD_i\hat r_{j\to i}.
 \label{eq:swap}
\end{equation}
The edit keeps the receiver's $\delta_a$, $q_i$, $\rho_i$, endpoint displacement norm, and gain fixed in exact arithmetic. The implemented bfloat16 state must satisfy the prespecified numerical tolerances. This is an endpoint patch: the window is not rerun to generate a new natural endpoint.

Each swap is compared with four random orientation controls matched to its rotation angle, chord/update length, and absolute patched-state norm. The latter is matching \emph{between the swap and its random controls}; it is not a claim that absolute state norm equals the unedited state's norm. The donor and all controls must be jointly feasible. No outcome-conditioned dropping or seed replacement is allowed. These controls address the possibility that a swap works merely because it rotates by a larger angle, applies a larger edit, or changes state scale differently from an unrelated direction.

Let $F\leftarrow S$ denote a failed receiver with a successful donor and $S\leftarrow F$ the reverse. Define the control-adjusted margin contrast for receiver $i$ by
\begin{equation}
 C_i=m_i^{\mathrm{swap}}-\frac{1}{K}\sum_{k=1}^{K}m_i^{\mathrm{random},k},
 \qquad B_i=\frac{C_{F\leftarrow S,i}-C_{S\leftarrow F,i}}{2}.
 \label{eq:swapcontrast}
\end{equation}
The primary signed bidirectional effect is the equal-item mean of $B_i$. These quantities are \emph{swap minus matched random}, not swap minus natural. Categorical rescue and harm are reported separately.

\begin{table}[htbp]
\centering\small
\caption{Qwen main-window orientation swaps. Margin columns are control-adjusted contrasts from \cref{eq:swapcontrast}. Categorical columns are the actual swap outcomes on failed and safe receivers. These are numerically eligible, selected mixed-outcome cohorts, not population failure rates. Archived result-summary provenance is recorded in \cref{app:provenance}.}
\label{tab:swap}
\begin{tabular}{lrrrrrr}
\toprule
Dataset & Paired bases & $C_{F\leftarrow S}$ & $C_{S\leftarrow F}$ & Mean $B$ & Rescue & Harm\\
\midrule
MMLU &44&$+1.201$&$-2.204$&$+1.702$&75.0\%&34.1\%\\
ReClor &7&$+0.875$&$-1.446$&$+1.161$&71.4\%&28.6\%\\
\bottomrule
\end{tabular}
\end{table}

\subsubsection{MMLU direction sensitivity}
The completed orientation run contains 1,935 edited forward results. At L52--54, 44 evaluation items are bidirectionally eligible. The failed-receiver swap-minus-random contrast is $+1.201$, while the safe-receiver contrast is $-2.204$, yielding signed bidirectional effect $+1.702$ (\cref{tab:swap}). Their archived bootstrap ranges are [0.757, 1.664], [$-2.719$, $-1.754$], and [1.368, 2.048], respectively; the raw interval configuration remains part of the archive-reconciliation note in \cref{app:provenance} rather than a newly recomputed coverage claim.

Actual swaps rescue 33 of 44 failures (75.0\%). The failed-receiver random controls retain failure in approximately 61.4\%, corresponding to approximately 38.6\% rescue, not 61.4\% rescue. Swapping a failed donor direction into safe receivers induces failure in 34.1\%, compared with approximately 1.1\% for the matched random controls. The control-adjusted margin contrast is positive on 38 of 44 failed receivers and negative on all 44 safe receivers.

Because the scalar coordinates are matched, this result demonstrates behavioral sensitivity to orientation beyond those scalar summaries in the eligible cohort. It does not identify a global safe direction: donors are matched within an item, and the transplant is expressed in the receiver's coordinates. Nor does it establish that a naturally occurring internal module computes the patched endpoint.

\subsubsection{Fixed-rule ReClor replication}
The later ReClor run inherits the Qwen windows and intervention rules. Its 604 evaluated variants contain 15 failures across 11 mixed-outcome items. Numerical screening yields 76 eligible direction/window groups and 380 edited forward results; seven items are bidirectionally eligible at the primary L52--54 window.

The failed-receiver contrast is $+0.875$ with nominal 97.5\% interval [0.179, 1.482]; the safe-receiver contrast is $-1.446$ [$-1.804$, $-1.063$]. Their signed bidirectional effect is $+1.161$ [0.813, 1.549], positive on all seven primary items. Actual swaps rescue five of seven failed receivers and cause two of seven safe receivers to fail; the corresponding safe-receiver random controls fail once in 28 draws. This supports fixed-rule cross-dataset direction sensitivity, but seven independent items do not establish precise population effects or broad task universality. The cohort's prior use also prevents calling it an untouched-stimulus confirmation.

\subsubsection{Longitudinal--transverse magnitude interventions}
\label{sec:qrho}
Orientation swapping does not answer whether longitudinal or transverse \emph{magnitude} has greater causal leverage. For a receiver $i$, the natural-dose factorial design instead holds its incoming vector and transverse orientation fixed while independently replacing longitudinal and transverse scalars:
\begin{equation}
 \delta_b(q^*,\rho^*)=(1+q^*)\delta_{a,i}
                  +\rho^*D_i\hat r_i,
 \quad q^*\in\{q_i,q_j\},\quad\rho^*\in\{\rho_i,\rho_j\}.
 \label{eq:factorial}
\end{equation}
All four cells are kept or excluded together. Donor orientation is not transplanted in this magnitude experiment. It deliberately changes $G$ and may change absolute state norm; it must not be described as another fixed-gain test.

The equal-budget design perturbs the longitudinal and transverse axes by the same absolute step $\eta D_i$, with $\eta=0.10$. Its donor-directed central finite response for scalar $k\in\{q,\rho\}$ is
\begin{equation}
 R_{k,i}=\operatorname{sign}(k_j-k_i)
       \frac{m_{k,+,i}-m_{k,-,i}}{2}.
 \label{eq:central}
\end{equation}
This is an antisymmetrized finite response in margin units. It is neither the derivative divided by step size nor the actual one-sided donor-directed margin change. Applying the same signed bidirectional aggregation, the archived MMLU main-window comparison retains 17 jointly eligible items and gives a transverse-minus-longitudinal contrast of $+0.2243$ (reported bootstrap range [0.0790, 0.3787]).

The direction-specific responses qualify that aggregate. On failed receivers, the transverse response is approximately $+0.401$ and the longitudinal response $-0.044$. On safe receivers, they are $-0.103$ and $-0.099$, respectively. The relative transverse advantage is therefore primarily visible on the failure side in this small feasible cohort. These values must not be converted into rescue rates or cited as proof that $q$ has no causal role. A large natural transverse-energy association does not imply a uniformly larger transverse effect for every receiver.

\subsubsection{Finite precision can make a comparison unestimable}
The ReClor magnitude/factorial plan has no geometrically and numerically feasible groups under its frozen rules, and no edited factorial results. Its corresponding primary magnitude comparison is therefore \emph{unavailable}, not a zero effect. Small longitudinal natural differences can fall below faithfully realizable bfloat16 edits, and the protocol forbids enlarging their dose or relaxing tolerances after inspection. The orientation and magnitude studies have different eligible cohorts; a successful orientation replication cannot stand in for the unperformed magnitude comparison. Natural-dose interaction estimates not recovered from the archived reports are not supplied as completed results in this revision.

\subsection{Gemma Entry Interventions and Matched Mirror Controls}
\label{sec:gemma-causal}
\subsubsection{Changing the entry changes the unpatched downstream state}
The Gemma experiments test a link different from Qwen endpoint correction. They intervene at L38 and measure the naturally computed L45 state, with L45 never directly patched. The cohort comprises 48 selected variants from 41 base items, evenly divided into 24 historical failures and 24 safe cases. Contracting the L38 state toward its own clean counterpart progressively reduces downstream displacement and improves the mean correct-answer margin.

For full entry replacement on the 24 failures, mean $D_{45}$ reduction is 11.924738, mean gold-margin gain is 0.855469, and four failures are rescued (16.7\%). Thus changing the entry state causally affects both the later representation and behavior. The much weaker discrete recovery than Qwen's endpoint replacement is not a controlled architecture comparison: the layer, intervention target, and cohorts differ.

Replacing only the last-token entry state leaves substantial later clean--variant difference. This shows that the last-token L38 vector alone is insufficient to determine the clean L45 trajectory. It does not uniquely identify contextual reinjection, its pathway, or $D_{45}$ as a causal mediator. The intervention changes more than one possible downstream decision-relevant feature.

\subsubsection{Empirical directions and shared descriptive geometry}
Arbitrary equal-norm random and orthogonal entry updates strongly disrupt previously safe Gemma answers. Such controls are therefore a limited test of selective, naturalistic correction. We additionally use clean-restoring directions from other examples. Same-family donors are much less destructive and can reduce downstream displacement, although the target's own direction gives stronger mean suppression.

Across 48 correction vectors of dimension 5,376, excluding same-base pairs, mean within-family cosine is 0.472478 and mean cross-family cosine is 0.335542. There is shared directional structure together with substantial family-specific structure. This is not evidence for a low-dimensional stabilizing subspace, a universal restoring vector, or a causal role for every shared component. Moreover, own-clean correction minimizes post-entry clean distance at a fixed update radius, confounding an own-versus-donor comparison.

\subsubsection{Mirror construction and what it matches}
Let $e_i=(h_{38,i}^C-h_{38,i}^T)/\norm{h_{38,i}^C-h_{38,i}^T}$ be the receiver's own correction axis and $u$ a unit empirical donor direction. Define
\begin{equation}
 u^{\mathrm{mir}}=2(e_i^\top u)e_i-u.
 \label{eq:mirror}
\end{equation}
The component \emph{along} the clean-correction axis is unchanged; only the orthogonal component is reversed. For update radius $sD_i$, empirical and mirror edits have equal norm, equal projection onto $e_i$, and equal resulting clean distance:
\begin{equation}
 \norm{h_{38,i}^T+sD_i u-h_{38,i}^C}^2
 =D_i^2\bigl(1+s^2-2s e_i^\top u\bigr).
 \label{eq:mirrordistance}
\end{equation}
The same expression holds for $u^{\mathrm{mir}}$. In actual bfloat16 runs these equalities are enforced only within prespecified numerical tolerances, not as exact real-arithmetic identities.

This design changes orientation at the \emph{entry}. It matches entry geometry but does \emph{not} hold the naturally computed downstream $D_{45}$ or gain fixed. Consequently, it tests the sufficiency of the matched entry scalars for propagation, not the sufficiency of a predictor that already observes the resulting downstream gain. Qwen's fixed-gain endpoint swap tests a distinct question.

\begin{table}[htbp]
\centering\small
\caption{Gemma L38 empirical--mirror entry experiment. Entries are mean reductions in naturally computed, unpatched $D_{45}$. Positive values indicate suppression. Each row averages donor conditions within receiver before averaging receiver cases.}
\label{tab:mirror}
\begin{tabular}{llrr}
\toprule
Receiver cohort & Donor family relation & Empirical & Mirror\\
\midrule
Failed ($n=24$) & Same family &2.221730&3.277838\\
Failed ($n=24$) & Cross family &0.924500&2.170946\\
Safe ($n=24$) & Same family &1.923202&3.696537\\
Safe ($n=24$) & Cross family &0.813023&2.571439\\
\bottomrule
\end{tabular}
\end{table}

\subsubsection{Downstream suppression does not imply categorical recovery}
The completed matched experiment contains 48 targets and 816 intervention rows, with exact no-op checks passing on all 48 targets. Its frozen primary contrast is cross-family empirical minus mirror, averaging donors within a target before the equal-target-case mean. Across all cases, that contrast in $D_{45}$ reduction is $-1.502431$; the mirror suppresses more than the empirical-donor average on 38 of 48 targets. \Cref{tab:mirror} shows the pattern in both receiver cohorts and both donor-family relations.

Discrete recovery does not follow the same ordering. On failed receivers, the own-direction condition rescues 3/24; same-family empirical and mirror conditions each rescue 5/96 donor-conditioned trials; cross-family empirical conditions rescue 5/96 and mirrors 3/96. These donor trials are not 96 independent base items. The own-direction condition here is also not the full-entry-replacement condition that rescued 4/24 in the preceding dose study.

This is a substantive negative result for the expectation that empirically observed donor directions should be privileged stabilizers once their entry geometry is matched. It is positive evidence for orientation-sensitive downstream processing, but not for a particular shared stabilizing subspace. The divergence between displacement suppression and discrete rescue further limits an account based on contraction alone. Gemma now supplies causal evidence; what remains absent is a completed identical replication of every Qwen orientation, magnitude, and norm-localization experiment.

\subsection{Finite Readout Effects and the Limits of a Single-Position Replay}
\label{sec:finite_reclor}

The preceding interventions establish directional sensitivity, but leave two quantitative questions open: how well does a local gradient describe a finite patch, and how closely does that patch reproduce the natural prompt contrast? We tested these questions separately on the frozen 707-rewrite ReClor set. The Qwen3-32B run completed 601 clean-correct pairs from 213 original questions; 106 pairs were skipped because the clean answer was incorrect. Fourteen rewrites failed, covering 11 original questions. The model, non-thinking prompt format, bfloat16 eager computation, and L52--54 window were inherited. These are reused stimuli rather than an untouched confirmation cohort.

For each pair, let $j_C=\arg\min_{j\ne y}m_j(x)$ be the clean question's worst competitor, fixed before intervention, and write $f(v;C)$ for its gold-versus-$j_C$ margin when the final prompt-token L54 state is $v$ in the clean input computation. This fixed-competitor response differs from recomputing the minimum gold margin after each edit. With $d_b=h_b^T-h_b^C$, define
\begin{align}
 C_b &= f(h_b^C+d_b;C)-f(h_b^C;C), &
 J_b &= \nabla_v f(h_b^C;C)^\top d_b, \label{eq:finite_clean_replay}\\
 P &= m_{j_C}(T(x))-m_{j_C}(x), &
 R_{\mathrm{ctx}} &= P-C_b. \label{eq:finite_prompt_gap}
\end{align}
Here $C_b$ is a local endpoint replacement with the rest of the clean computation retained, whereas $P$ changes the entire prompt. The identity $P=J_b+(C_b-J_b)+R_{\mathrm{ctx}}$ is bookkeeping for these specified comparisons, not a decomposition into uniquely identified neural pathways.

We additionally approximate the response change along the fixed clean-context interpolation using
\begin{equation}
 I_b=\int_0^1 \nabla_v f(h_b^C+t d_b;C)^\top d_b\,dt.
 \label{eq:finite_path_replay}
\end{equation}
The numerical run uses nested trapezoidal grids with eight and sixteen subintervals, retaining all three competitor responses and the realized endpoint update. In exact differentiable arithmetic, the integral equals the corresponding finite response. Numerical agreement must nevertheless be checked for the implemented bfloat16 computation.

\begin{table}[tbp]
\centering
\small
\caption{Finite readout checks on 601 ReClor pairs from 213 original questions. Entries are base-equal means of absolute fixed-competitor margin quantities; nominal 95\% intervals use 2,000 original-question bootstrap resamples. The local patch error and the natural-prompt discrepancy measure different targets.}
\label{tab:finite_replay}
\begin{tabularx}{\textwidth}{Yrr}
\toprule
Quantity & Mean & 95\% interval\\
\midrule
Natural prompt effect $|P|$ & 1.1451 & $[1.0267,1.2774]$\\
Local finite patch effect $|C_b|$ & 0.8995 & $[0.8031,1.0056]$\\
One-point local error $|C_b-J_b|$ & 0.1753 & $[0.1577,0.1947]$\\
Integrated local error $|C_b-I_b|$ & 0.1300 & $[0.1220,0.1392]$\\
Fine--coarse integral difference & 0.00162 & $[0.00146,0.00180]$\\
Remaining-context discrepancy $|P-C_b|$ & 0.3942 & $[0.3595,0.4326]$\\
\bottomrule
\end{tabularx}
\end{table}

Path integration improves the mean absolute local-patch approximation by 0.0453 margin units, or approximately 25.8\% of the one-point error (\cref{tab:finite_replay}). It does not numerically recover the whole finite effect: the remaining defect of 0.1300 is substantially larger than the 0.00162 fine--coarse difference. Small refinement differences therefore do not by themselves validate numerical completeness. Moreover, the natural-prompt discrepancy exceeds the one-point local error by a base-equal mean of 0.2188, with nominal 95\% interval $[0.1847,0.2522]$. Better integration of this particular patch cannot eliminate the difference between a partial-state replay and changing the full input.

Response alignment also improves reference-assisted failure ranking in this run. Under simultaneous held-family and five-fold held-item exclusion, the mean held-family AUROC is 0.9536 for clean margin, 0.9735 for $(M,\Din,\Gw)$, 1.0000 after adding signed clean-context endpoint erosion, and 0.9994 after adding the analogous entry-gradient projection. The endpoint increment over the scalar predictor is 0.0265, with nominal 95\% base-bootstrap interval $[0.0114,0.0449]$. Only fourteen failures support these estimates. More fundamentally, these features require each question's clean reference, paired activations, gold-defined margins, and clean readout gradients. They are strong reference-assisted diagnostics, not evidence for a detector that operates without clean references or labels. Their success at L52--54 also does not retroactively validate the earlier L32 boundary approximation, which used a different layer and local response.

\subsection{Natural Window Propagation and an Initially Negative Component Test}
\label{sec:window_components}

We next intervened inside the window and let its downstream state be computed, rather than directly constructing its endpoint. A last-token entry-state/context exchange at L52 yields four computations: clean or rewritten input context, each with its natural or exchanged final-token entry state. L54 is measured without patching it. This design tests whether the selected entry vector is sufficient to reproduce the natural endpoint difference; it does not isolate a unique circuit.

Across the same 601 pairs, the base-equal mean relative gap between the entry-induced and natural endpoint differences is 0.4774, with nominal 95\% interval $[0.4689,0.4861]$. A single last-token entry vector therefore leaves substantial endpoint information unspecified. The mean absolute fixed-competitor margin response is 0.7843 for the context arm and 0.5402 for the entry arm; their paired base-equal difference is 0.2441 $[0.1702,0.3167]$. The four-cell signed interaction averages 0.0223 $[-0.0027,0.0483]$, while its absolute magnitude averages 0.2023 $[0.1873,0.2174]$. Opposite interaction signs can cancel in the mean; the factorial identities themselves are not estimates of mediation fractions.

The module scan then considers the final-token attention and MLP outputs in blocks 53 and 54. For component $c$, the repair edit adds its same-item clean-minus-rewrite output difference to the rewritten computation; the reverse edit subtracts that difference in the clean computation. The suffix is recomputed. Controls use four random directions in the same component coordinates, matched for update norm and resulting component norm. Numerical eligibility is decided from natural activations and realized bfloat16 geometry before edited responses are read. Unlike the earlier transverse endpoint swaps, these edits do not hold endpoint displacement, window gain, or full residual-state norm fixed.

Selection uses discovery questions and the two source families, excluding both the target family and its held-out original questions. The first rule selects one component by its all-pair, base-equal bidirectional margin response. Its frozen validation evaluates all clean-correct, untied target rewrites, not only natural failures. The 286 requested target cases yield 244 eligible pairs from 85 original questions; only three are natural failures.

\begin{table}[tbp]
\centering
\small
\caption{The original all-pair component validation is negative. Scores average matched-random-adjusted repair and damage effects on the fixed clean-worst-competitor margin. Each family uses one selected module. Nominal 98.33\% original-question bootstrap intervals allocate a nominal 95\% familywise level across the three primary contrasts.}
\label{tab:component_negative}
\begin{tabularx}{\textwidth}{Ycrrr}
\toprule
Held-out family & Output & Pairs & Score & 98.33\% interval\\
\midrule
Grammatical restructuring & L54.attn & 82 & $-0.0694$ & $[-0.1717,0.0309]$\\
Irrelevant insertion & L53.attn & 80 & $0.0297$ & $[-0.0284,0.0912]$\\
Lexical paraphrase & L54.attn & 82 & $0.0943$ & $[-0.1507,0.3456]$\\
\bottomrule
\end{tabularx}
\end{table}

All three primary intervals include zero (\cref{tab:component_negative}). The pooled descriptive score is 0.0184, with nominal 95\% interval $[-0.0606,0.1027]$, and no strict failure is rescued among the three untreated failures. Component restoration nevertheless reduces exit displacement by a base-equal mean of 5.7189 $[5.2490,6.2460]$, corresponding to a mean relative reduction of 10.49\% $[9.87\%,11.18\%]$. Reducing displacement is therefore not sufficient to establish an average answer benefit in this cohort. Answer-changing and label-permutation controls were not supplied at this stage, so this run alone does not establish task-specific selectivity.

\subsection{A Failure-Conditioned Follow-Up Rule}
\label{sec:failure_followup_rule}

An exploratory diagnosis of the discovery scan explains why a uniformly positive all-pair effect is an unsuitable assumption. The common discovery cohort contains 357 rewrites from 128 questions, including eleven strict failures from eight questions. L54.attn has a raw bidirectional score of only 0.0452 over the full cohort, but 1.1797 $[0.8125,1.5469]$ among the untreated failures. Its repair edit strictly rescues two of eleven failures; an additional tied answer is counted only by an argmax convention. These discovery effects are not adjusted by matched random controls.

The same cleanward edit can also undo a beneficial rewrite. For correct rewrites whose natural margin decreases, the L54.attn bidirectional score is 0.3751 $[0.2952,0.4591]$ over 154 pairs from 92 questions. For correct rewrites whose margin increases, it is $-0.4000$ $[-0.5114,-0.2972]$ over 160 pairs from 96 questions. These outcome-conditioned strata are descriptive and their intervals are nominal; they are not independent confirmations. They show why averaging all rewrites can combine effects with opposite behavioral signs.

This diagnosis motivates a new selection rule: rank a component by the smaller of its base-equal repair and damage effects among source-family natural failures, require positive effects in both source families, and abstain when the fixed coverage gates fail. The rule was frozen before the later identity-disjoint evaluation. Its target is a failure-conditioned directional effect, not a replacement estimate for the preceding all-pair test. We retain that earlier negative result and evaluate the revised rule separately below.

\subsection{Frozen component validation on disjoint original items}
\label{sec:fresh-validation}

The component experiments permit a stronger test than selecting an effective direction after observing a target failure: freeze the intervention site and selection rule, acquire different original items without consulting target outcomes, and evaluate the resulting paired directions. We fixed the Qwen L52--54 window and the final-prompt-token L54 attention output. The failure-conditioned rule selected this component using other-family discovery examples before the new target cohort was evaluated. This follow-up rule was motivated by the preceding all-cohort experiment; its failure-conditioned estimand does not retroactively replace that experiment's primary analysis.

\panel{Acquisition and intervention.}
The labelled source contained 500 ReClor items. After excluding the 250 previously used items, the remaining 250 formed the new cohort. Original-item identifiers and clean-input token hashes were disjoint from the supplied historical collection. This establishes separation within the study, not absence from model pretraining. Fixed generation slots produced 746 candidates; 702 passed AI semantic screening: 239 lexical paraphrases, 214 grammatical restructurings, and 249 irrelevant insertions. The generator and reviewer used the same local Qwen3.8-27B checkpoint, so the review is correlated AI screening rather than independent human certification. Collection and screening did not select candidates by target-model failure or intervention success.

For item $i$, let $c_i^C$ and $c_i^T$ be the natural clean and rewrite attention outputs at the measured token, and set $d_i=c_i^C-c_i^T$. The repair adds $d_i$ to the rewrite component, while reverse intervention adds $-d_i$ to the clean component; both recompute the suffix. Four random directions per intervention match the own update's norm and the resulting component-output norm, with numerical eligibility determined before edited responses are inspected. These are component-level constraints, not fixed full-window gain or residual-state geometry.

In this experiment the response is the margin against the strongest competitor on the unedited clean input, fixed before intervention: $m_i^{\mathrm{fix}}=z_y-z_{j_i^*}$, where $j_i^*=\arg\max_{j\ne y}z_j(x_i)$. For arm $a$, define
\begin{equation}
 S_i^{a}(m)=\frac{1}{2}\left[
 \overline m_i^{\mathrm{repair},a}-m_i^T
 +m_i^C-\overline m_i^{\mathrm{damage},a}\right].
 \label{eq:component-bidirectional}
\end{equation}
The bar averages the arm's draws, with one draw for the own paired direction. The primary contrast is $S_i^{\mathrm{own}}-S_i^{\mathrm{random}}$. Positive values combine greater repair with greater reverse damage than matched random edits. Natural failure and strict rescue are assessed separately using the minimum margin over \emph{all} wrong answers: failure requires a negative minimum and rescue a positive minimum. A fixed-competitor improvement alone need not restore the correct answer.

\begin{table}[t]
\centering\small
\caption{Frozen L54 attention intervention on disjoint ReClor original items. The effect is the random-adjusted bidirectional score from \cref{eq:component-bidirectional}, using the clean-fixed competitor margin. Intervals are nominal 98.33\% original-item percentile bootstrap intervals allocated over three family comparisons. Failure counts equal distinct failure-item counts within each family; items can overlap across families. All family counts are below the frozen twenty-failure-item interpretive threshold.}
\label{tab:fresh-component}
\setlength{\tabcolsep}{5pt}
\begin{tabularx}{\linewidth}{Yrrrr}
\toprule
Family & Eligible & Failures & Effect [interval] & Strict rescue\\
\midrule
Grammar &187&6&$0.583\;[0.271,0.872]$&2/6\\
Irrelevant insertion &219&10&$1.200\;[0.571,2.127]$&5/10\\
Lexical paraphrase &215&11&$1.369\;[0.818,1.970]$&6/11\\
\bottomrule
\end{tabularx}
\end{table}

\panel{Results and scope.}
There were 621 eligible rewrites on 220 clean-correct original items; 81 rewrites came from 30 clean-incorrect items. The eligible set contained 27 strict failures from 21 original items. The own correction strictly rescued 13/27 failures (48.1\%), whereas the 108 matched random draws rescued none. The additional argmax-correct case was a boundary tie and is excluded from strict rescue. All 27 failures had a positive random-adjusted bidirectional score. Family-level effects were positive (\cref{tab:fresh-component}), with original-item averaging and 2,000 original-item bootstrap resamples. Nevertheless, each family had only six to eleven failure items, below the frozen interpretive threshold of twenty. The threshold is a practical reporting rule, not a power calculation, and the displayed small-sample intervals do not convert this result into precise population confirmation.

The all-clean-correct reference analysis remains materially weaker: its family-level scores were $-0.017$, $0.106$, and $0.091$ for grammar, irrelevant insertion, and lexical paraphrase, respectively; all three adjusted intervals included zero. The positive failure-conditioned result therefore supports a local effect in a subset, not a generally beneficial intervention. A small label-permutation control also exposed a specificity problem. Thirteen unique controls were clean-eligible, although repetition across three plans produced 39 records. All thirteen were naturally correct after relabelling, but one became strictly incorrect after the old-reference correction: its new correct letter was B, while the patch restored the old letter D and reduced the minimum correct margin from $5.75$ to $-0.50$. This motivated the expanded order controls below.

\subsection{Option-order and answer-letter controls}
\label{sec:order-controls}

The expanded experiment asks whether the own paired direction contains information beyond a generic direction associated with the correct answer letter, and whether the result depends on the option layout. It reuses the frozen rewrite cohort after the preceding results; it is a targeted mechanism diagnostic, not another original-item-independent replication.

\panel{Four orders and matched donors.}
We rendered all four cyclic option orders, consistently remapping the gold letter. A fixed pre-inference text screen excluded 21 candidates from eight items for option-order-dependent language. The retained 681 rewrites from 242 original items yielded 2,724 conditions. The intervention site and numerical rules remained frozen. In each current order, the own arm used that order's clean-minus-rewrite component difference. Four random directions and four same-letter donor directions were compared with it; different-letter donors formed a secondary control. The 357 eligible source pairs came from the earlier discovery collection, had strictly correct, untied clean answers, and were original-item-disjoint from recipients. Donors were from another rewrite family and shared either the recipient's \emph{current} gold letter or a different letter. Assignments were fixed by deterministic hashing before recipient inference. Donor source rewrites were not required to fail, nor were their source corrections required to have a positive effect.

To isolate orientation under local component constraints, let $h$ be the recipient component baseline, $u$ its own repair or reverse update, and $d$ a signed donor direction. Set $e=h/\norm{h}$ and $d_\perp=d-\ip{d}{e}e$. For nondegenerate directions, the control update is
\begin{equation}
 v(d)=\ip{u}{e}e+
 \norm{u-\ip{u}{e}e}\,\frac{d_\perp}{\norm{d_\perp}}.
 \label{eq:component-sphere-control}
\end{equation}
Thus $\norm{v}=\norm{u}$ and $\norm{h+v}=\norm{h+u}$ in exact arithmetic. Here orthogonality is relative to the component state $h$, rather than the incoming residual displacement used in the earlier window decomposition. Both intervention directions and all control arms were planned before edited forward passes, with joint eligibility checked after casting and at the actual hooks. Maximum recorded between-arm update- and component-norm discrepancies were 0.367\% and 0.138\%, below the fixed 1\% tolerance.

This experiment uses the minimum gold margin $M=\min_{j\ne y}(z_y-z_j)$ directly in \cref{eq:component-bidirectional}, rather than the preceding experiment's clean-fixed competitor. Its two primary contrasts are $S^{\mathrm{own}}-S^{\mathrm{random}}$ and $S^{\mathrm{own}}-S^{\mathrm{same\ letter}}$. They are averaged within original item and bootstrapped over original items, keeping repeated orders and draws together. Their effect sizes must not be pooled with the fixed-competitor scores above.

\begin{table}[t]
\centering\small
\caption{Option-order controls among strict natural failures whose original item is clean-correct in every order. Effects are differences in bidirectional \emph{minimum}-margin score. Nominal 99.17\% percentile intervals allocate error over six planned family contrasts; 2,000 original-item resamples are used. The small per-family event counts make these exploratory estimates. Draws within a recipient and orders of the same item are correlated.}
\label{tab:order-controls}
\setlength{\tabcolsep}{4pt}
\begin{tabularx}{\linewidth}{Yrrr}
\toprule
Family; failures/items & Own $-$ random [interval] & Own $-$ same letter [interval] & Rescue\\
\midrule
Grammar; 9/8 &$0.631\;[0.271,0.936]$ &$0.627\;[0.444,0.827]$&3/9\\
Irrelevant; 13/11 &$0.894\;[0.634,1.244]$ &$0.826\;[0.552,1.187]$&6/13\\
Lexical; 15/13 &$1.938\;[1.145,3.024]$ &$1.934\;[1.083,3.330]$&7/15\\
\bottomrule
\end{tabularx}
\end{table}

\panel{Own-direction advantage under the current order.}
All 2,724 conditions were processed: 2,415 were complete, 295 had an incorrect clean answer, and fourteen had a clean tie; no condition failed numerical eligibility. Requiring strict clean correctness in all four orders gave a common cohort of 189 original items and 538 rewrites, or 2,152 conditions. Their natural outcomes were 2,107 strictly correct, 37 strict failures, and eight boundary ties. The 37 failures represented 32 distinct rewrites and 23 original items. Of the preceding experiment's 27 original-order failures, fifteen remained in this common cohort and twelve were excluded because another order's clean answer was not strictly correct. The new failure count therefore does not represent 37 additional independent failures.

All six planned contrasts were positive (\cref{tab:order-controls}), and the own bidirectional score exceeded both random and same-letter donor means in all 37 failure conditions. Own corrections strictly rescued 16/37 conditions (43.2\%), compared with 0/148 random draws, 10/148 same-letter draws, and 4/148 different-letter draws. These control denominators count four correlated interventions per recipient, not independent questions. Restricting the analysis to changed orders left 22 failures from fifteen original items: own corrections rescued 8/22, compared with 0/88 random and 5/88 same-letter draws. This subgroup is a supplementary diagnostic. Family event counts of eight, eleven, and thirteen remain below twenty, and percentile-bootstrap tail estimates are coarse at 2,000 resamples.

\panel{Context dependence limits the common-direction claim.}
The same rewrite's failure was rarely stable across layouts: among the 32 rewrites failing in any order, 28 failed in exactly one order, three in two, one in three, and none in all four. The all-common-clean own-minus-same-letter scores were also less uniform than the failure-conditioned results: $-0.038$ for grammar, $0.123$ for irrelevant insertion, and $0.103$ for lexical paraphrase; the grammar interval included zero. Thus failure-subset effects do not establish a universal mean advantage.

A separate reference-mismatch diagnostic literally replaced the changed-order component with its original-order clean value. It made the prediction select the old gold letter in 400/1,614 changed-order conditions (24.8\%), versus 8/1,614 naturally and 3/1,614 after the current-order paired correction. Among 1,585 naturally strict-correct changed-order conditions, this stale reference produced 368 strict errors and 59 ties. However, its update norm was a median 5.15 times the matched own update. It is consequently evidence of reference mismatch, not an equal-budget comparison proving a pure letter mechanism. Same-letter donors likewise do not exhaust all possible answer-readout explanations.

Together, these results show that the natural paired direction carries local, context-dependent causal information beyond the tested component-norm constraints and same-letter donor controls. They sharpen the earlier endpoint evidence by intervening in a frozen attention component and testing a concrete answer-letter alternative. They do not establish a single layout-invariant direction: every current-order correction recomputes its paired clean reference. The interventions still require that reference, and neither a clean-free detector nor generally selective robustness follows. Independent blinded semantic review and direct tests of unchanged-direction transfer remain open.

\subsection{Bridge to the monitorability axis: consequential relative to what?}
The fragility experiments establish that equal scalar geometry can conceal consequential directional differences. They do not assign a response-independent sign to the transverse component. A direction that restores a gold margin in one computation need not increase an arbitrary target score in another. The next axis therefore measures both the direction and its alignment with an explicit readout. Its reference is a matched lexical prompt rather than a clean question; its patch is an additive transverse intervention rather than a fixed-gain orientation swap. What transfers is the decomposition and the question of directional consequence, not the original classifier or an assumption that all transverse changes are harmful.

\section{Axis II: Measuring Influence Not Disclosed in CoT}
\label{sec:monitor}
\label{sec:crossbehavior}
The second axis asks what can be measured when generated reasoning does not acknowledge an external cue. It changes the behavioral question from \emph{answer fragility} to \emph{readout influence conditional on recorded CoT}. This is our operational treatment of the hidden-CoT motivation, not a claim to recover private reasoning. The defining distinction is between disclosure in text, a prompt-level score contrast, and a localized causal patch effect.

\subsection{Nondisclosure is an observation, not an intent label}
MonitorBench supplies cue-answer tasks and conditions including direct concealment \citep{wang2026monitorbench}. We use recorded direct-conceal continuations and separate blind AI-assisted disclosure review. A \texttt{DISCLOSED} label requires visible acknowledgement of the external cue's role in inference, verification, search, or answer selection. Answer agreement alone does not count as disclosure. \texttt{NOT\_DISCLOSED} therefore describes the recorded text; it does not certify covert reliance, deceptive intent, or a particular internal reasoning process. The one-rollout monitor used to complete the upstream pipeline is not the label source for this analysis.

The cross-domain review report contains 690 continuations: 342 labeled disclosed and 348 not disclosed. The latter form the reported fixed-CoT analysis cohort across 82 bases. We do not interpret 348/690 as the prevalence of hidden reasoning. It is a text-label proportion in this collected sample, under this operational review criterion. No intervention-effect sign was supplied as a disclosure ground truth.

\subsection{Fixed reasoning and exact-target lexical control}
Let $c$ be a recorded generated CoT and $y_{1:m}$ a fixed target answer sequence. We replay the same $c$ and the same teacher-forced target prefix under two prompts. $S$ presents the standard answer cue; $L$ contains the exact same target text but frames it as unrelated text. The direct-conceal condition supplies the recorded continuation; the subsequent scoring contrast is \emph{standard cue versus lexical control}, not direct-conceal versus standard.

For prompt $P\in\{S,L\}$, define
\begin{equation}
 \ell_P(c,y)=\frac{1}{m}\sum_{t=1}^m
    \log p(y_t\mid P,c,y_{<t}),\qquad
 U_{\sema}=\ell_S(c,y)-\ell_L(c,y).
 \label{eq:semantic}
\end{equation}
The historical name $U_{\sema}$ denotes this specified contrast, not a purified estimate of all semantic influence. Matching the target text addresses a simple lexical explanation for target-score increases; it does not remove every difference in framing or establish semantic equivalence of $S$ and $L$. Unlike the task-preserving rewrites of Axis I, the two cue prompts deliberately assign different roles to the same target text.

Holding $c$ fixed removes variation in the \emph{recorded reasoning text} from this scoring comparison. It does not freeze the hidden activations induced by that text under different prompts. Nor does it show whether the cue originally caused a step in $c$, measure the probability of generating $c$, or estimate the distribution of newly sampled answers. The continuation is reused counterfactually rather than regenerated. An influence on this fixed-CoT readout can be positive or negative and need not cause a discrete answer flip.

\subsection{Target-position geometry and signed sequence attribution}
At positions predicting target tokens, let
\begin{equation}
 \delta_{a,t}=h^S_{a,t}-h^L_{a,t},\qquad
 \delta_{b,t}=h^S_{b,t}-h^L_{b,t},\qquad W=(52,54).
 \label{eq:monitor-states}
\end{equation}
The Qwen window is inherited rather than searched again on math/health outcomes. At each valid position, \cref{eq:decompose} gives $q_t$ and $r_t$ such that $\delta_{b,t}=(1+q_t)\delta_{a,t}+r_t$. Let $g_t$ be the gradient of the same mean target-sequence response at the lexical-reference computation. Define
\begin{align}
 J_{\trans,\mathrm{seq}}&=\sum_t g_t^\top r_t,\label{eq:jseq}\\
 J_{\longi,\mathrm{seq}}&=\sum_tg_t^\top[(1+q_t)\delta_{a,t}],\qquad
 J_{\mathrm{total},\mathrm{seq}}=J_{\longi,\mathrm{seq}}+J_{\trans,\mathrm{seq}}.
 \label{eq:jseq-components}
\end{align}
The target-length normalization is already part of the differentiated mean response; the sum over states is not an additional average over target tokens. Target-prediction alignment matters: the final prompt-token probe used for next-token multiple-choice labels is not automatically adequate for a multi-token target after generated reasoning.

Neither $U_{\sema}$ nor $J_{\trans,\mathrm{seq}}$ is a hidden-CoT class label. $U_{\sema}$ compares whole prompts in a fixed continuation, whereas $J_{\trans,\mathrm{seq}}$ projects a selected part of the endpoint difference onto a local gradient. Agreement between them is an empirical relation, not an identity. A later patch separately tests whether this selected direction has the expected local effect.

\subsection{Localized causal intervention and controls}
Write $H_b^L=(h^L_{b,t})_t$ for the lexical-reference tuple and $R=(r_t)_t$ for its paired transverse difference. The implemented intervention is
\begin{equation}
 H_b^L\longmapsto H_b^L+\alpha R,\qquad
 C_{\trans}(\alpha)=\phi(H_b^L+\alpha R)-\phi(H_b^L),
 \label{eq:readoutpatch}
\end{equation}
where $\phi$ is the mean target-sequence response with the lexical prompt, recorded CoT, target prefix, and unpatched layer-$b$ states fixed. This $C_{\trans}$ is the $\Delta\ell_{\trans}$ reported in the source analysis. \Cref{eq:derivative} gives $C'_{\trans}(0)=J_{\trans,\mathrm{seq}}$ for the idealized differentiable reference computation; actual finite-dose responses are measured, not assumed. Here $\alpha=0$ is a no-op. This differs from the semantic contraction parameter in \cref{eq:alpha}, where $\alpha=1$ is a no-op.

The frozen doses are $\{-0.5,0,0.25,0.5,0.75,1\}$. At $\alpha=1$, the protocol also tests the longitudinal tuple, the total difference tuple, and a norm-matched random-orthogonal control. These controls ask whether the transverse patch differs from other specified edits. They do not turn it into the fixed-gain swap of Axis I: gain and absolute state norm are not held constant by adding $R$ to the lexical state.

Patching only selected endpoint positions leaves other cue-related differences unpatched. Thus even a total-difference patch at these positions need not equal $U_{\sema}$. A ratio of patch effect to prompt contrast is not a mediation proportion, and the effects of longitudinal and transverse finite patches need not add exactly because later computation can be nonlinear.

\subsection{Analysis unit and four noninterchangeable quantities}
Repeated continuations share a base question. For base $i$ with $n_i$ retained continuations, the analysis first averages each response, for example
\begin{equation}
 \bar J_i=\frac{1}{n_i}\sum_{k=1}^{n_i}J_{ik},\qquad
 \bar C_i(\alpha)=\frac{1}{n_i}\sum_{k=1}^{n_i}C_{ik}(\alpha).
 \label{eq:base-aggregation}
\end{equation}
The reported sign accuracy compares the sign of these base means, not a vote over continuation-level signs. Math and health are analyzed separately and receive equal weight in macro estimates. Bootstrap resampling draws bases within domains, with 5,000 resamples; 348 continuations are not 348 independent observations.

\begin{table}[htbp]
\centering\small
\caption{Operational quantities for the hidden-CoT-motivated axis. None alone establishes unobserved reasoning or intentional concealment. The analysis conditions on nondisclosed continuations rather than training a disclosure classifier from geometry.}
\label{tab:monitor-quantities}
\begin{tabularx}{\linewidth}{@{}L{0.23\linewidth}YY@{}}
\toprule
Quantity & What it measures & What it does not establish\\
\midrule
\texttt{NOT\_DISCLOSED} & No explicit cue attribution in recorded CoT. & Cue reliance, intent, or absence of internal influence.\\
$U_{\sema}$ & Standard-minus-lexical target score with $c$ fixed. & A fresh answer flip, complete cue use, or influence on the original generation of $c$.\\
$J_{\trans,\mathrm{seq}}$ & Local signed readout projection of transverse differences. & A task-independent harmfulness score or guaranteed finite-dose effect.\\
$C_{\trans}(\alpha)$ & Response to a specified internal patch in lexical-reference replay. & Total prompt effect, unique mediation, or an endogenous concealment mechanism.\\
\bottomrule
\end{tabularx}
\end{table}

This design makes the monitorability question testable without collapsing text and computation into one label. On the nondisclosed cohort, we ask whether the prompt contrast covaries with the internal attribution, whether the attribution predicts a finite patch response, and whether that response has a common population sign. These are three different tests. Their distinction is central to the results that follow.

\section{Axis II: Causal Effects and Cross-Domain Tests}
\label{sec:monitor-results}
The evidence proceeds from logic development to frozen math/health evaluation and then to a post hoc signed-effect analysis. Reordering the paper around two axes does not change that chronology. In particular, the refined signed hypothesis is not substituted for the original positive-mean success rule.

\subsection{Logic development: a positive response, with incomplete specificity}
The reported logic causal analysis includes 93 continuations from 25 bases. At $\alpha=1$, the transverse patch changes the mean target-token log-probability response by $+0.007032$, with the saved primary 95\% interval [0.001410, 0.012946]. The nonnegative-dose slope is $+0.007074$ [0.001394, 0.012936]. At $\alpha=-0.5$, the effect reverses to $-0.004787$ [$-0.007817$, $-0.002053$]. The positive-dose means increase from approximately 0.002119 to 0.004117, 0.005740, and 0.007032 at doses 0.25, 0.5, 0.75, and 1.

The transverse-minus-longitudinal contrast is $+0.002425$ [0.000301, 0.004784]. Transverse minus norm-matched random is $+0.006003$, but its interval [$-0.0000454$, 0.011876] includes zero. The development result therefore supports a graded response without establishing every desired specificity comparison. Base-level Spearman correlation between $J_{\trans,\mathrm{seq}}$ and the finite transverse patch effect is 0.966154. These are development observations, not untouched confirmation.

\subsection{Frozen math/health association with the prompt contrast}
The new-domain analysis retains 348 nondisclosed continuations from 82 bases: 20 math and 62 health bases. The frozen observational primary relates base-averaged $J_{\trans,\mathrm{seq}}$ to $U_{\sema}$. Spearman correlations are 0.330827 on math and 0.526126 on health. The equal-domain Fisher aggregate is 0.433554, with domain-stratified base-bootstrap 95\% interval [0.132692, 0.686577]. Both domain point directions are positive, so the specified observational gate passes.

This relation extends beyond predicting a perturbation's Euclidean size: it connects an internal, response-aligned quantity with the fixed-CoT score difference between cue framings. Nevertheless, correlation does not establish that the selected transverse positions mediate the whole prompt effect, or that the original generated CoT covertly used the cue. The subsequent intervention test has a different estimand.

\subsection{The original positive-mean causal criteria fail}
\label{sec:monitor-frozen-failure}
The frozen core rule additionally requires a positive average transverse $\alpha=1$ effect, a positive nonnegative-dose slope, and a negative effect at $\alpha=-0.5$. Full confirmation also requires a positive transverse-minus-random contrast. The $\alpha=1$ effect must have positive point direction in both domains. These causal requirements are not met (\cref{tab:monitor-causal}); both core and full confirmatory flags remain false.

\begin{table}[htbp]
\centering\small
\caption{Original frozen Qwen math/health causal criteria on 82 bases. Effects are equal-domain macro changes in mean target-token log probability, with saved domain-stratified bootstrap 95\% intervals. The last row is an additional protocol contrast, not an independent primary success gate.}
\label{tab:monitor-causal}
\begin{tabular}{@{}lrr@{}}
\toprule
Causal quantity & Macro estimate & 95\% interval\\
\midrule
Transverse patch, $\alpha=1$ & $-0.002178$ & [$-0.012131$, 0.009869]\\
Nonnegative-dose slope & $-0.002203$ & [$-0.012234$, 0.009882]\\
Transverse patch, $\alpha=-0.5$ & $-0.000769$ & [$-0.008935$, 0.004652]\\
Transverse minus random, $\alpha=1$ & $-0.001956$ & [$-0.012497$, 0.010982]\\
\midrule
Transverse minus longitudinal & $-0.002286$ & [$-0.006876$, 0.002146]\\
\bottomrule
\end{tabular}
\end{table}

At $\alpha=1$, the point effect is $-0.004399$ on math and $+0.0000428$ on health. The failure is therefore not merely a wide aggregate interval around two positive domain means. The logic development pattern does not establish a uniformly cue-promoting transverse response in new domains. \Cref{fig:monitor-effects} puts development and frozen-domain mean effects on the same response scale without treating the two cohorts as a paired comparison.

\begin{figure}[htbp]
\centering\includegraphics[width=\linewidth]{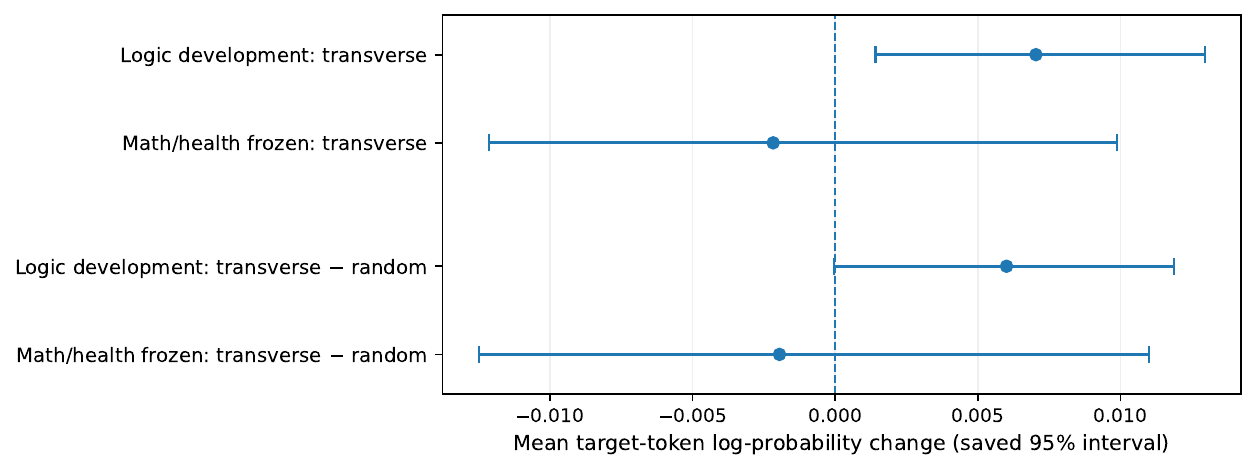}
\caption{Logic development versus frozen math/health mean-effect tests. Points and intervals reproduce saved analysis summaries, not a new bootstrap. Logic uses 25 base items; math/health uses an equal-domain macro estimate over 82 bases. The original positive-mean and specificity claims do not pass the held-domain test. This figure is not a direct statistical test of the difference between the two cohorts.}
\label{fig:monitor-effects}
\end{figure}

\subsection{Signed attribution and finite-step patch effects}
\label{sec:monitor-signed}
The signed local account asks whether different examples are pushed in predictably different directions, rather than whether every domain has a positive mean. On the completed math/health outputs, base-level Spearman correlation between $J_{\trans,\mathrm{seq}}$ and the actual transverse $\alpha=1$ patch effect is 0.993985 on math and 0.933318 on health. Their equal-domain Fisher aggregate is 0.979806 [0.940296, 0.999776]. This close derivative-to-intervention relation is distinct from the more moderate 0.433554 relation to $U_{\sema}$.

The total-position patch has a separate association with $U_{\sema}$: 0.416541 on math and 0.761880 on health, with aggregate 0.618202 [0.356109, 0.815947]. \Cref{fig:monitor-links} distinguishes the three comparisons. Agreement with a patch that uses the same reference, response, and positions is a localized causal validation; it is not equivalent to predicting the entire cue-prompt contrast. Strong agreement can reflect local smoothness as well as properties of the chosen direction.

\begin{figure}[htbp]
\centering\includegraphics[width=\linewidth]{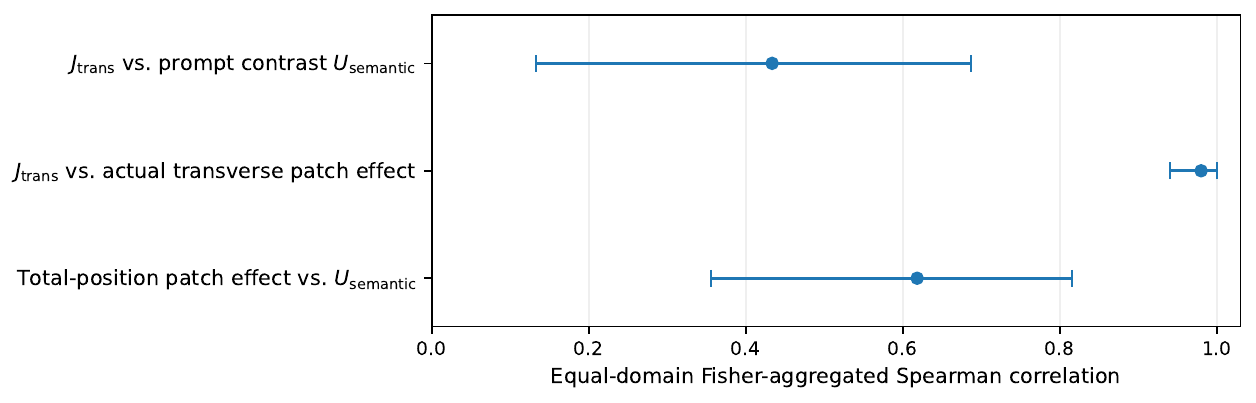}
\caption{Three distinct links in the fixed-CoT analysis. Values are equal-domain Fisher aggregates of base-level Spearman correlations, with saved domain-stratified 95\% intervals. Only the attribution--prompt-contrast relation is the frozen observational primary. Attribution--patch agreement and total-patch--prompt-contrast association answer different causal-scope questions and must not be reported interchangeably.}
\label{fig:monitor-links}
\end{figure}

A subsequent signed analysis uses
\begin{equation}
 \widehat s_i=\operatorname{sign}(\bar J_{\trans,i}),\qquad
 s_i=\operatorname{sign}(\bar C_{\trans,i}(1)).
 \label{eq:signrule}
\end{equation}
The rule correctly predicts 19/20 math-base signs and 54/62 health-base signs, or 73/82 overall (\cref{tab:signed}). The pooled accuracy is 89.02\% with saved bootstrap 95\% interval [81.71\%, 95.12\%]. Equal-domain macro accuracy is 91.05\% [83.63\%, 96.77\%]; pooled balanced accuracy is 89.93\%, and pooled Matthews correlation is 0.7694. The actual signs are not balanced: 27 bases have positive effects and 55 have negative effects, motivating the additional balanced metric rather than accuracy alone.

\begin{table}[htbp]
\centering\small
\caption{Post hoc prediction of finite transverse patch-effect signs on nondisclosed bases. These are not hidden-reasoning detection scores. The independent unit is the base, after averaging its continuations. ``Positive/negative'' counts describe the actual patch effects.}
\label{tab:signed}
\begin{tabular}{@{}lrrrrr@{}}
\toprule
Domain & Bases & Positive/negative & Correct signs & Accuracy & Balanced acc.\\
\midrule
Math &20&10/10&19/20&95.00\%&95.00\%\\
Health &62&17/45&54/62&87.10\%&87.45\%\\
Pooled &82&27/55&73/82&89.02\%&89.93\%\\
\bottomrule
\end{tabular}
\end{table}

This refinement is descriptive. The sign threshold is zero, but choosing signed-effect accuracy after seeing the mean-effect failure is a change of inferential target. It does not retroactively pass the frozen test. Values close to zero also require numerical-stability checks in a replication; the reported zero-threshold analysis did not introduce a new dead zone or exclude small effects retrospectively. A Gemma L38--45 signed replication has been prepared, but no completed effect estimates are included here.

\subsection{What the second axis adds to the paper}
On CoT labeled nondisclosed, the analysis provides a measured prompt contrast, a response-aligned internal attribution, and causal responses to specified endpoint patches. The association with the prompt contrast transfers to the two tested domains, while the directional derivative closely tracks finite patch responses of either sign. Thus the second axis contributes more than another failure-rate benchmark: it identifies which \emph{kind} of internal influence is measurable when explicit acknowledgement is absent, and distinguishes that question from a universal increase in target probability.

The conclusion remains conditional on fixed-CoT replay and the chosen lexical reference. We have not shown that a monitor would fail on those same examples, that geometry outperforms a textual monitor, or that every nondisclosed continuation contains hidden cue use. No effect-size threshold or labeled hidden-reasoning detection task was evaluated. The supported claim is narrower and testable: direction-sensitive geometry can help characterize localized, signed readout influence in a nondisclosed-CoT cohort. It complements the answer-fragility axis without treating its outcome as another correct-to-incorrect flip.

\section{Related Work}
\label{sec:related}
\panel{Linguistic robustness and representation change.}
RUPBench supplies heterogeneous perturbation families \citep{wang2024rupbench}, while latent adversarial paraphrasing connects prompt robustness with representation movement \citep{fu2025same}. Work on hidden instability in vision--language models distinguishes stable answers from stable internal representations \citep{wani2026same}. Our semantic-fragility axis asks why some movement changes an answer while other movement is tolerated, and tests the information omitted by scalar gain. None of the internal analyses replaces semantic validation of the input pairs.

\panel{CoT monitorability and nondisclosed influences.}
MonitorBench studies chain-of-thought monitorability \citep{wang2026monitorbench}. Our second axis uses its cue-answer setting but distinguishes a disclosure label from fixed-continuation cue influence and from a localized causal effect. The exact-target lexical contrast, target-position attribution, and endpoint patching define additional analysis outcomes. They are not benchmark monitor scores or evidence that private reasoning has been recovered. The relation to monitorability is the gap between what the recorded rationale acknowledges and what controlled readout measurements can establish.

\panel{Causal localization and activation steering.}
Causal tracing of linguistic realization choices shows that controlled rewrites can change judgments and that patching effects can concentrate at particular layers and positions \citep{huang2026realization}; late localization alone is not our novelty claim. Activation-patching conclusions depend on design \citep{zhang2024activation}, and activation steering uses contrastive directions to alter behavior \citep{stolfo2025steering}. We distinguish fixed-scalar endpoint orientation swaps, entry-geometry-matched mirrors with unpatched downstream readouts, and signed target-position patches. These are diagnostic experiments, not one interchangeable steering method or a deployment-ready defense.

\panel{Margins and local sensitivity.}
Classical robustness analyses relate margins to sensitivity and can derive certified bounds \citep{tsuzuku2018lipschitz,muthukumar2023local}. Our observed gain ratio is not a Lipschitz bound. Likewise, the directional-derivative identity is elementary calculus rather than a novel theorem about transformers. The empirical questions are where direction-sensitive effects occur, how they transfer, and how far a local derivative predicts a finite intervention. The new finite-replay experiment tests one of these limits directly: numerical path integration improves a specified local patch approximation, while the larger natural-prompt discrepancy remains. The two research axes answer different parts of these questions.

\section{Discussion: What Connects Fragility and Monitorability?}
\label{sec:discussion}
\subsection{A common sequence of questions, not a common error label}
The first axis starts from an answer failure and asks how it arises. The second starts from absent cue acknowledgement and asks what influence can nevertheless be measured under a specified replay. They share a sequence of questions: how do paired representations differ, how does that difference change through a window, and what consequence does its direction have for the chosen response? They do not share semantic equivalence assumptions, reference prompts, response variables, or statistical labels.

This distinction is the conceptual role of the hidden-CoT axis. An answer label summarizes a decision; a rationale is a text that may or may not attribute its sources. Neither substitutes for a controlled account of internal influence. The additional measurements do not reveal a complete hidden reasoning process, but they make a narrower aspect of monitorability accessible to intervention. The common framework is therefore broader than a catalogue of answer-flipping rewrites and narrower than a universal theory of covert reasoning.

\subsection{Why the sign refinement matters for both axes}
The initial fragility account emphasized insufficient contraction relative to clean reserve. Natural trajectories, fixed-scalar orientation swaps, and the Gemma mirror result show why that description is incomplete. The CoT axis makes the omitted sign especially clear: a transverse component can raise or lower the specified target score. Its magnitude cannot encode whether either change is desirable or whether a gold-answer boundary is crossed.

The failed math/health positive-mean criterion must remain a negative result. A near-zero mean with heterogeneous signs is consistent with an accurate signed predictor, but that mathematical possibility does not itself establish such a predictor; the reported finite-patch comparisons provide the separate empirical evidence. Because the signed target was adopted after the failed confirmation, it remains a refined hypothesis awaiting independent replication. This separation prevents a sequence of hypothesis changes from being presented as one successful preregistered theory.

\subsection{What is shared across tasks and models}
There are three levels of convergence. The longitudinal--transverse decomposition provides common \emph{measurement coordinates}. Held-family prediction and frozen ReClor evaluation support useful \emph{predictive relationships}, though not uniform calibration or scalar sufficiency. Matched interventions support \emph{causal sensitivity} to specified directions. These levels do not identify the same circuit, learned operator, or hidden-state basis across architectures.

In particular, the CoT result is not zero-shot transfer of the MMLU logistic predictor: it uses a response-specific gradient and a different input contrast. The inherited Qwen window links the experiments without establishing that it is optimal for all reasoning settings. Gemma supplies semantic-fragility prediction and entry-intervention evidence, but no completed CoT signed-effect replication in this draft. ``Shared'' consequently refers to the experimental description that survives these distinctions, not to evidence that every setting has the same failure mechanism.

\subsection{Local directional information is not an invariant failure vector}
The new component studies sharpen the semantic claim in three ways. First, they retain an unsuccessful all-pair validation: restoring a component can reduce endpoint displacement without a demonstrable average answer benefit. Second, a failure-conditioned rule informed by that result is frozen before evaluation on a disjoint set of original items. The positive result is evidence for that refined target, rather than retrospective success of the initial estimand. Third, matched same-letter donors test a concrete alternative to input-specific directional information. Their weaker effects show that sharing the gold answer letter does not account for the own-direction advantage under the tested component geometry.

The order manipulation also exposes an important boundary. Most affected rewrites fail in only one layout, and each intervention uses the clean reference for its current layout. Repeated success of recomputed paired directions therefore differs from transfer of an unchanged vector. Same-letter donors do not eliminate every answer-readout mechanism, and the larger stale-reference intervention cannot be treated as a size-matched contrast. The evidence supports context-dependent local directional information; stable generation of failure through one layout-invariant direction remains unestablished. Nor do local response fits close the gap between a finite state patch and a natural input change.

\subsection{From a diagnostic to a monitor}
A usable hidden-influence monitor would require a defined detection target, suitable negative cases, a decision threshold, and evaluation of false positives and false negatives under realistic access constraints. Our sign task has a different target: the direction of a particular patch effect on an already selected nondisclosed cohort. Its 89.02\% accuracy cannot be converted into a monitor's detection accuracy. A high derivative-to-patch correlation also does not prove superiority over reasoning-text monitoring; no such head-to-head comparison is supplied.

The present methods require paired reference computations and internal activations. The new semantic gradient features additionally use gold-defined readout gradients; the CoT axis uses readout gradients and fixed continuations. Even the near-perfect reference-assisted semantic ranking is consequently not a clean-free detection result. They are controlled diagnostics rather than a deployable system. Their value is to separate claims that surface observations conflate: failure versus stability, disclosure versus influence, and localized response versus total prompt effect. The two axes strengthen that separation without erasing the experiments' limitations.

\section{Limitations}
\label{sec:limitations}
\panel{Stimulus validity and different tasks.}
The strongest rewrite analyses concern MMLU and ReClor multiple-choice questions. Published families and unchanged labels do not certify semantic preservation. The historical audit is retrospective, overlaps fitting data, and uses preliminary AI-assisted labels; ReClor is AI-screen-passed rather than human-certified. In the new acquisition the generator and reviewer share a checkpoint; no independent blinded semantic audit is included in the completed results. In the CoT axis, $S$ and $L$ deliberately change the role of the cue and are not semantic-equivalence interventions. The two outcomes must not be pooled as instances of the same error.

\panel{Chronology, selection, and precision.}
Historical windows were selected outside nested prediction-fold selection. Qwen's endpoint parameterization was developed after its initial ReClor result, with fitting restricted to historical MMLU. The historical ReClor transverse-swap intervention reuses stimuli and has seven primary bidirectional items; the MMLU equal-dose cohort has 17. The new component rule is informed by a negative all-pair test, then frozen for a disjoint 250-item cohort. That cohort has only 21 failure items; its reused option-order analysis has 23. These totals overlap across families, and all per-family failure counts are below the frozen twenty-item interpretive threshold. Cyclic orders and donor draws do not create independent original items. Mixed-outcome and feasibility-selected cohorts do not estimate population-average effects. For CoT, logic is development data, math/health fail the original causal success rule, and signed-effect accuracy is post hoc. Nondisclosure-conditioned continuations are not an unselected population sample.

\panel{Fixed-CoT inference is not hidden-reasoning recovery.}
Blind disclosure labels are operational text judgments, not direct labels of cue reliance or intent. Teacher-forced target scoring holds the recorded CoT fixed but changes its internal processing with the prompt. It does not establish what caused the original rationale, predict free-running answer changes, or demonstrate that an external monitor was fooled. The target need not be correct, and positive influence is not synonymous with harm. Only localized patch-effect sign prediction is quantified; no hidden-CoT detector is evaluated.

\panel{Partial states, nonlinear response, and numerical precision.}
Most semantic interventions patch only the final prompt-token state; CoT patches affect only specified target-prediction positions. Neither isolates a unique mediator of the whole prompt effect. Finite-dose longitudinal and transverse effects need not add, and local gradient agreement can reflect smoothness of a closely matched response. bfloat16 matching is approximate. The new component controls match update and resulting component norms, not every full-window scalar; recorded norm checks do not imply equally small vector errors. Joint numerical feasibility avoids outcome-selective edits but restricts coverage; zero feasible ReClor magnitude groups mean the comparison is unavailable, not null. Near-zero sign stability remains a replication consideration.

\panel{Cross-model coverage and applications.}
Gemma has semantic entry-contraction, empirical-direction, and mirror results, not an identical replication of every Qwen intervention. Its separate CoT signed-effect study has no completed result included here. MASC's retained exploratory procedure labels use static judging without execution, and behavior-blind search is limited to Qwen with an automatic semantic gate. These additional applications do not supply independent confirmation of the second axis.

\panel{Evidence recovery.}
This revision integrates completed experimental records; it performs no additional target-model inference. New finite-replay, component, fresh-item and order-control archives are directly available, with saved-output checks and independent fresh/order reanalyses. Some historical values remain inherited from the supplied manuscript. Some historical intervention values are retained from prior summaries because their original archives were not re-openable. The logic primary summary and current cross-domain/sign summaries are directly readable; the differing evidence levels are recorded in \cref{app:provenance}. Reorganizing the paper does not resolve those raw-data audit requirements or turn engineering tests into scientific results.

\section{Conclusion}
\label{sec:conclusion}
Heterogeneous input changes admit useful shared descriptions without sharing a universal causal direction. Clean reserve, displacement, and local gain rank semantic failures across families; matched interventions show that scalar summaries omit consequential orientation. New finite-replay checks improve local quantitative description through path integration while retaining a substantial gap to the full natural input effect. The initial component validation gives no clear all-cohort benefit. A revised rule frozen before new original items repairs a subset of failures, and expanded order controls show an own-direction advantage over matched random and same-letter donor directions.

The resulting semantic claim is local and context-dependent. Most failing rewrites fail in only one tested option order, and current-order repair requires a newly computed clean reference. Shared predictive relationships and repeated directional sensitivity therefore do not yet establish one order-invariant failure vector, a clean-free detector, or a general repair method. In the complementary fixed-CoT study, the readout-relative association transfers across domains, the frozen positive-mean causal criterion fails, and signed-effect prediction remains post hoc. Together, these results motivate a direction-sensitive account tied to explicit references, responses, and intervention conditions, with independent confirmation and broader transport still required.

\clearpage
\bibliographystyle{plainnat}
\bibliography{references}
\clearpage
\appendix
\section{Historical Family and Depth Statistics}
\label{app:depth}
\begin{table}[htbp]
\centering\small
\caption{Historical failure rates by family on model-specific clean-correct items. Each retained base has all six variants. These published transformations are not all independently verified semantic equivalents.}
\begin{tabular}{lrr}
\toprule
Family & Qwen3-32B & Gemma 4 31B\\
\midrule
Checklist&7.02\%&7.24\%\\
Compound&5.01\%&4.91\%\\
It-cleft&9.52\%&9.81\%\\
Red-herring&7.77\%&5.37\%\\
Stress&7.27\%&5.84\%\\
Wh-cleft&11.78\%&10.05\%\\
\bottomrule
\end{tabular}
\end{table}
\begin{table}[htbp]
\centering\small
\caption{Historical strict mean LOFO AUROC by probe endpoint. ``Compact'' uses clean margin, the preceding displacement, and local gain. The first probe has no local-gain predecessor.}
\begin{tabular}{llccc}
\toprule
Model & Layer & Margin + current $D$ & Margin + trajectory & Compact\\
\midrule
Qwen&32&0.870&0.869&--\\
&40&0.869&0.869&0.869\\
&48&0.874&0.875&0.877\\
&50&0.885&0.900&0.906\\
&52&0.895&0.921&0.923\\
&54&0.918&0.949&0.953\\
&56&0.918&0.948&0.916\\
\midrule
Gemma&30&0.790&0.794&--\\
&38&0.837&0.847&0.846\\
&45&0.899&0.910&0.910\\
&47&0.897&0.909&0.898\\
&49&0.896&0.908&0.894\\
&51&0.893&0.906&0.897\\
&52&0.886&0.907&0.897\\
\bottomrule
\end{tabular}
\end{table}

\section{Earlier Trajectory Diagnostics and Causal Details}
\label{app:historical-details}
The first-order boundary ratio at 50\% Qwen depth obtains AUROC 0.683. The archived $R>1$ crossing rate is zero in five families and 0.25\% in it-cleft. Mid-layer displacement alone gives 0.462 AUROC, while clean margin gives 0.869. Adding observed mid-to-final gain raises the three-variable diagnostic to 0.907.

A denser trajectory experiment probes 12.5, 25, 37.5, 50, 62.5, 75, and 87.5\% depth. The pre-final trajectory score reaches AUROC 0.937. Ridge prediction of late log transmission from the earlier trajectory gives mean held-family correlation 0.800 and $R^2=0.610$; a two-stage failure model using the predicted transmission reaches AUROC 0.922. These are theory-development diagnostics, not additional independent confirmations of the final selected window.

\begin{table}[htbp]
\centering\small
\caption{Historical Qwen failure-side contraction slopes and absolute-norm layer details. Slope intervals are archived bootstrap 95\% intervals. The layer experiment's binary late-window contrasts are not significant under its small fixed update budget.}
\begin{tabular}{lrr}
\toprule
Family & Contraction slope & 95\% interval\\
\midrule
Checklist&4.04&[2.39, 6.09]\\
Compound&3.03&[2.38, 3.66]\\
It-cleft&5.47&[4.21, 6.87]\\
Red-herring&4.54&[3.29, 5.92]\\
Stress&4.71&[3.51, 5.92]\\
Wh-cleft&5.59&[4.41, 6.93]\\
\bottomrule
\end{tabular}
\par\smallskip
\begin{tabular}{rrrr}
\toprule
Layer & Margin improvement & Rescue & Mean fraction corrected\\
\midrule
40&0.023&2.78\%&0.500\\
48&0.038&1.39\%&0.339\\
50&0.155&4.17\%&0.244\\
52&0.269&8.33\%&0.178\\
54&0.465&9.72\%&0.116\\
56&0.417&6.94\%&0.089\\
\bottomrule
\end{tabular}
\end{table}

The original direction-control run uses 72 failed variants from 52 base items, with eight random and eight orthogonal draws per variant. Its baseline replay has no reported mismatches. The absolute-norm layer run uses the same 72-case design. Baseline agreement checks and small realized norm discrepancies support the implementation of these interventions; they do not remove the scientific selection and reference-state limitations discussed in the main text.

\section{Context Dependence of a Single-Position Probe}
\label{app:context}
The geometry is exact in the chosen residual coordinates, but a single position is not generally a complete state. Let $v$ be the patched last-token state at endpoint $b$, and let $K^C$ and $K^T$ stand for the remaining clean and transformed computation states or context needed by the suffix. Write a scalar suffix response as $\psi_b(v;K)$. The natural response difference admits the exact telescoping decomposition
\begin{align}
 \psi_b(h_b^T;K^T)-\psi_b(h_b^C;K^C)
  ={}&\big[\psi_b(h_b^T;K^T)-\psi_b(h_b^C;K^T)\big]\nonumber\\
    &+\big[\psi_b(h_b^C;K^T)-\psi_b(h_b^C;K^C)\big].
 \label{eq:contexttelescoping}
\end{align}
The first bracket is a single-position intervention contrast with the transformed context held fixed. The second is a remaining-context contrast at a shared last-token state. Where the required hybrid computation is defined and smooth, the first bracket can be locally expanded as $g^\top\delta_b+\mathcal R$. Substituting \cref{eq:decomp} then yields longitudinal and transverse first-order terms, plus a nonlinear remainder. Neither the context term nor that remainder is assumed to vanish in this paper.

This decomposition formalizes a limitation; it is not an empirical attribution of all residual effects to a named pathway. It also explains why full replacement of one position need not recover the clean answer. Multiple-choice minimum margin is piecewise defined by competing options; local per-option gradients do not themselves certify preservation of every margin under a finite edit.

\section{Magnitude Factorial Estimands}
\label{app:factorial}
For the natural-dose four-cell design, let $m_{00}$ denote the receiver scalars, $m_{10}$ the donor longitudinal scalar only, $m_{01}$ the donor transverse magnitude only, and $m_{11}$ both donor scalars, with receiver orientation fixed. A standard factorial accounting would report
\begin{align}
 E_q&=\tfrac12[(m_{10}-m_{00})+(m_{11}-m_{01})],\\
 E_\rho&=\tfrac12[(m_{01}-m_{00})+(m_{11}-m_{10})],\\
 E_{q\rho}&=m_{11}-m_{10}-m_{01}+m_{00}.
\end{align}
These define finite-intervention contrasts, not intrinsic model derivatives or a decomposition of natural causal responsibility. Numerical four-cell feasibility and any aggregation must be established for this design itself. No missing natural-dose interaction estimate is reconstructed from the equal-dose central responses. In particular, the reported 17-item equal-dose result does not establish that all corresponding natural-dose factorial cells were valid, and ReClor's zero feasible magnitude groups cannot yield any of these effects.

\section{MonitorBench Outcome and Estimand Checklist}
\label{app:monitor-estimands}
\begin{table}[htbp]
\centering\small
\caption{Distinct comparisons in the fixed-CoT monitorability axis. All displayed aggregate correlations use equal-domain Fisher aggregation of base-level Spearman correlations. These quantities must not be interchanged in an abstract or figure.}
\begin{tabularx}{\linewidth}{@{}XrX@{}}
\toprule
Comparison & Estimate & Interpretation\\
\midrule
$J_{\trans,\mathrm{seq}}$ versus $U_{\sema}$ &0.4336&Frozen observational association; criterion passed.\\
$J_{\trans,\mathrm{seq}}$ versus actual transverse patch effect &0.9798&Signed local attribution-to-intervention agreement.\\
Total patch effect versus $U_{\sema}$ &0.6182&Association of a different intervention with the prompt contrast.\\
Mean transverse $\alpha=1$ patch effect &$-0.002178$&Frozen positive-mean causal criterion failed.\\
Sign-rule pooled accuracy &89.02\%&Post hoc signed patch-effect prediction, not hidden-reasoning detection.\\
\bottomrule
\end{tabularx}
\end{table}

At the base level, the signed classifier's pooled confusion counts are 48 true negatives, 7 false positives, 2 false negatives, and 25 true positives. There are 27 positive and 55 negative actual patch effects. The math counts are 9, 1, 0, and 10; the health counts are 39, 6, 2, and 15. A zero sign threshold does not eliminate the post hoc choice of this analysis target. Values very close to zero also require attention to numerical stability when the rule is replicated.

\section{Prospective Search Beyond a Fixed Perturbation Taxonomy}
\label{sec:search}\label{app:search}
The historical experiments ask whether geometry organizes known perturbation failures. A separate Qwen search tests whether the frozen compact score remains useful before candidate answer outcomes are observed. The result is positive for instance-level enrichment but negative for discovering a new reusable linguistic family.

\subsection{Candidate-behavior-blind protocol}
The historical raw-feature score, imputation and standardization statistics, coefficients, L52--54 window, checkpoint, prompt format, and non-thinking configuration are fixed before behavior reveal. The discovery cohort contains 241 clean-correct MMLU questions absent from the historical fitting data. Their clean outcomes and margins are known; \emph{candidate} correctness is not used in generation, semantic gating, geometric scoring, beam selection, or lineage inspection.

A broad natural-rewrite generator is not instructed to produce the six historical named families. Candidates are gated for task-relevant content preservation, naturalness, omitted information, and new answer hints. The gate does not see candidate behavior or geometry. Within each item, the initial beam contains four candidates. Each of three recursive rounds requests up to six local descendants per beam parent, then gates and scores the descendants before retaining four. Two low-risk initial candidates per item are frozen as controls. The completed blind archive contains 25,691 scored candidates.

Surface-operation descriptions of blind lineages are clustered independently of hidden states and failure labels. Before answer reveal, a checkpoint decides whether any cluster supports a coherent, reusable rule distinct from the historical families. The archive, reveal set, and any optional rule are frozen by hash. All three rounds were completed and the manifest records no candidate behavior used in search. No family rule was frozen.

\subsection{Failure enrichment after frozen reveal}
The reveal set contains 4,335 unique rewrites with 648 failures (14.95\%). Three rewrites occupy two selection memberships, giving 4,338 memberships in \cref{tab:search}. The final beam fails on 20.02\% of candidates, compared with 2.49\% of low-risk controls: an absolute difference of 17.53 percentage points and relative enrichment of 8.04 times. The archived 5,000-resample item-cluster interval for the absolute difference is approximately [13.5, 21.7] percentage points.

\begin{table}[htbp]
\centering\small
\caption{Candidate outcomes revealed only after search freeze. Counts are selection memberships, not independent observations. Three duplicate memberships account for the difference between 4,338 memberships and 4,335 unique candidates.}
\label{tab:search}
\begin{tabular}{lrrr}
\toprule
Frozen selection group & Memberships & Failures & Failure rate\\
\midrule
Low-risk controls &482&12&2.49\%\\
Initial beam (round 0)&964&121&12.55\%\\
Beam after round 1&964&151&15.66\%\\
Beam after round 2&964&174&18.05\%\\
Final beam (round 3)&964&193&20.02\%\\
\bottomrule
\end{tabular}
\end{table}

The increase from initial to final beam is 7.47 percentage points, with archived item-bootstrap interval approximately [4.5, 10.6]. The last-round increment is smaller and its interval includes zero. This supports improvement over the whole search trajectory, not a separately significant gain at every round.

Across frozen-risk bins $[0,.05)$, $[.05,.10)$, $[.10,.25)$, $[.25,.50)$, $[.50,.75)$, and $[.75,1]$, observed failure rates are 0.36\%, 5.11\%, 11.21\%, 19.28\%, 39.54\%, and 57.40\%. The score thus orders previously unevaluated candidates along a steep failure gradient. Since selection intentionally enriches both risk extremes, these aggregate numbers are not calibrated prevalence estimates for an unselected rewrite population.

\begin{table}[htbp]
\centering\small
\caption{Prospective candidate ranking after outcome reveal. Predictor values were frozen before candidate correctness was known. The final beam is a restricted comparison after all three search rounds.}
\label{tab:searchauc}
\begin{tabular}{lrr}
\toprule
Predictor & Full frozen set AUROC & Final beam AUROC\\
\midrule
Frozen $(M,D_{52},G_{52\to54,\epsilon})$ score &0.905&0.886\\
Incoming displacement &0.797&0.772\\
Local gain &0.776&0.748\\
Small clean margin ($-M$)&0.740&0.709\\
Absolute token-length change &0.535&0.531\\
\bottomrule
\end{tabular}
\end{table}

The joint score remains informative within the final beam (\cref{tab:searchauc}), where all candidates have undergone the same search. Length change is close to chance. Compared with non-failures, revealed failures have lower mean clean margin (6.08 versus 10.78), larger incoming displacement (89.96 versus 62.03), and higher local gain (1.798 versus 1.580). These associations support the frozen search criterion; they do not establish the scalar score as causally sufficient.

\subsection{No new perturbation family was discovered}
Blind lineage analysis yields 12 surface clusters, but no clean cross-item rule. The strongest cluster enrichment is only approximately 1.08 times; leading clusters are content-specific or mix synonyms, voice changes, clarification, and generic restructuring. The family specification was therefore left empty before reveal. This run has no valid held-out family-versus-neutral confirmation experiment, and its successful candidates cannot now be grouped post hoc and called a prospectively discovered seventh family.

The split outcome matters. A frozen internal score can guide search toward fragile instances even when those instances do not collapse into a single surface category. That is consistent with a failure boundary cutting across linguistic forms, but it is also compatible with limitations of the proposal distribution or surface-clustering method. The experiment cannot distinguish those explanations. It remains Qwen-specific and relies on automatic semantic gating rather than exhaustive independent human annotation.

\section{Exploratory Procedure Steering in Code Generation}\label{app:masc}
Saved Gemma 4 31B generations from the MASC study \citep{hu2026masc} test whether a benign metaphorical skill shifts a model from its usual efficient procedure toward a lower-efficiency target procedure. Five no-skill generations define a procedure reserve
\begin{equation}
 M_{\mathrm{proc}}=\log\frac{n_{\mathrm{efficient}}+0.5}{n_{\mathrm{target}}+0.5}.
 \label{eq:proc}
\end{equation}
The primary window is fixed from Gemma's MMLU analysis at L38--45. The dataset has 77 records, 20 positive under the broad procedure-shift label. Mean gain is 0.889 for positives and 0.860 for negatives; gain alone obtains AUROC 0.668 with one-sided Mann--Whitney $p=0.013$. A fixed $(M_{\mathrm{proc}},D_{38},G_{38\to45,\epsilon})$ model gives AUROC 0.591 versus 0.386 for reserve alone.

The sample is small, calibration is poor, and labels come from offline static LLM judging rather than program execution or a replay of the original MASC judge. This is suggestive cross-behavior evidence only. It does not establish that generated programs remain correct, that runtime increases, or that the MMLU scalar law transfers without qualification.

\section{Reproducibility and Evidence Status}
\label{app:provenance}
The immediate editing basis is the supplied 35-page manuscript, \path{HeterPeturb__Zhibo_Hu_.pdf}. Its matching earlier LaTeX package supplies editable body sections, bibliography, and historical plot assets. The visible title and conceptual figure are retained from the supplied PDF. The abstract, introduction, chronology, discussion, limitations, and conclusion are revised to incorporate only the completed semantic-fragility experiments documented in the present study sequence. Existing MonitorBench methods and results are retained without introducing an additional research framework. This editorial revision performs no new target-model inference.

\panel{Directly available new evidence.}
The finite replay experiment is read from its complete 707-case archive and analysis summaries, including 601 completed cases and 106 clean-incorrect skips. Window tracing, component discovery/validation, and the failure-conditioned follow-up are read from their saved outputs, frozen plans, and diagnostics. The new-item validation and option-order controls are additionally checked against independent saved-logit reanalyses. The accompanying evidence summaries record paths, denominators, estimates, and uncertainty settings. The effect tables use archived estimates; differences and percentages explicitly described in the text are arithmetic summaries, not new model experiments.

The finite-patch table consistently uses the archived realized one-point-error interval. A separately bootstrapped version of the numerically identical error gives slightly different endpoints; the two are not mixed. The new-item experiment fixes the clean worst competitor for margin effects, whereas the option-order experiment uses minimum gold margin. They are deliberately not pooled. Strict rescue requires a positive minimum margin, so argmax-correct boundary ties are not counted as rescues. Bootstrap resampling keeps observations from an original item together, including repeated layouts and donor draws.

The source package includes the new finite-replay, component, follow-up, new-item, and order-control scalar summaries and their evidence ledger. Shared result archives omit large activation caches; saved-output reanalysis therefore does not independently remeasure every vector-level acceptance check. Recorded numerical checks are distinguished from a rerun with activations. No synthetic engineering test is presented as a pretrained-model scientific result. The prepared direction-transport implementation has no completed target-model result in this revision, and unfilled semantic-review templates are not evidence of an independent review.

\panel{Retained historical evidence.}
Historical MMLU prediction, Qwen endpoint dose, direction, and layer tests, transverse swaps, the original seven-item ReClor swap replication, Gemma entry/mirror results, prospective search, MASC, and fixed-CoT results are retained from the supplied manuscript. The inherited Qwen natural-window CSV and MonitorBench scalar files remain in the source package. Some older intervention figures were originally recovered from result summaries rather than complete per-condition archives. This revision does not recast them as newly recomputed results; exact historical cohort IDs, complete-plan checks, interval settings, numerical residuals, and natural-dose factorial coverage remain subject to reconciliation with their original records.

Two older bootstrap contrasts are not conflated with the mean-LOFO differences in \cref{tab:historical}. The preceding record reports Qwen improvement 0.097 [0.064,0.133] and Gemma improvement 0.145 [0.100,0.204], whereas subtracting the mean-LOFO columns gives 0.084 and 0.123. The available manuscript does not establish their aggregation mapping. We retain the specified mean-LOFO metrics without silently changing the bootstrap estimand. Likewise, older MMLU swap and equal-dose ranges retain their uncertainty qualifications; the historical ReClor protocol explicitly supplies nominal 97.5\% intervals.

The fixed-CoT positive-mean criterion remains failed, and its 73/82 signed-effect result remains post hoc. No unfinished Gemma CoT study is counted as confirmation. The prospective-search result still establishes instance-level enrichment without discovering a new reusable linguistic family. The unavailable original ReClor magnitude comparison remains unavailable rather than null. The present component evidence addresses a distinct, narrower set of interventions and does not replace these historical outcomes.

\section{Supplementary Reference-Assisted Readout Checks}
\label{app:reference_checks}
The finite-replay study also constructs directions aligned with a local readout gradient and directions orthogonal to it. These are calibration controls for response sensitivity, not naturally generated rewrite directions. At the fixed update radius $0.1\norm{d_b}$, 601 groups were requested, but only 65 groups from 56 original items passed joint bfloat16 norm and direction tolerances; 536 failed numerical eligibility. No dose enlargement or outcome-based seed replacement was used.

The aligned signed central response is 2.3828, with nominal 95\% original-item interval [2.2466,2.5352], while the orthogonal response is 0.0123 [$-0.0067$,0.0318]. The aligned-minus-absolute-orthogonal contrast is 2.3292 [2.1886,2.4710]. These results confirm sensitivity to directions deliberately constructed from the same readout gradient within the feasible cohort. They do not establish that natural rewrites follow those directions, and the low numerical coverage prevents treating them as representative of all 601 pairs. Passing update-norm and cosine checks also does not imply an equally small full-vector relative error.

The reference-assisted ranking in \cref{sec:finite_reclor} uses held-out-family and original-item exclusion for score fitting. That split does not remove the inference-time access to the clean computation, gold-defined margins, and gradients. Accordingly, near-perfect ranking with these features is reported as a controlled diagnostic. Neither these checks nor the component repairs constitute a detector evaluated without clean references.

\end{document}